%% file: example_paper.tex
\documentclass{article}

\usepackage{microtype}
\usepackage{graphicx}
\usepackage{subfigure}
\usepackage{booktabs} 

\usepackage{hyperref}
\usepackage[ruled,longend]{algorithm2e}

\usepackage[accepted]{icml2024}

\usepackage{amsmath}
\usepackage{amssymb}
\usepackage{mathtools}
\usepackage{amsthm}

\usepackage[capitalize,noabbrev]{cleveref}

\definecolor{LightGreen}{rgb}{0.8,1,0.8}
\definecolor{LightBlue}{rgb}{0.8,0.8,1}

\usepackage{color}
\usepackage{multirow}
\usepackage{enumitem}
\usepackage{bm}
\usepackage{colortbl}
\usepackage{float}
\usepackage{wrapfig}
\usepackage{makecell}
\usepackage{bbding}

\theoremstyle{plain}
\newtheorem{theorem}{Theorem}[section]

\theoremstyle{definition}
\newtheorem{definition}[theorem]{Definition}

\theoremstyle{remark}

\usepackage[textsize=tiny]{todonotes}

\newcommand{\Name}{A3S}

\def\##1\#{\begin{align}#1\end{align}}
\def\$#1\${\begin{align*}#1\end{align*}}

\def\knn{\mathrm{knn}}

\def\HH{\mathbb{H}}
\def\II{\mathbb{I}}
\def\PP{\mathbb{P}}

\icmltitlerunning{\Name: A General Active Clustering Method with Pairwise Constraints}

\begin{document}

\twocolumn[
\icmltitle{\Name: A General Active Clustering Method with Pairwise Constraints}



\icmlsetsymbol{equal}{*}

\begin{icmlauthorlist}
\icmlauthor{Xun Deng}{s1,comp}
\icmlauthor{Junlong Liu}{comp}
\icmlauthor{Han Zhong}{s2}
\icmlauthor{Fuli Feng}{s1}
\icmlauthor{Chen Shen}{comp}
\icmlauthor{Xiangnan He}{s1}
\icmlauthor{Jieping Ye}{comp}
\icmlauthor{Zheng Wang}{comp}
\end{icmlauthorlist}

\icmlaffiliation{comp}{Alibaba Group}
\icmlaffiliation{s1}{University of Science and Technology of China}
\icmlaffiliation{s2}{Peking University}

\icmlcorrespondingauthor{Zheng Wang}{wz388779@alibaba-inc.com}
\icmlcorrespondingauthor{Fuli Feng}{fulifeng93@gmail.com}

\icmlkeywords{Machine Learning, ICML}

\vskip 0.3in
]



\printAffiliationsAndNotice{}  

\begin{abstract}
Active clustering aims to boost the clustering performance by integrating human-annotated pairwise constraints through strategic querying. Conventional approaches with semi-supervised clustering schemes encounter high query costs when applied to large datasets with numerous classes. To address these limitations, we propose a novel \underline{A}daptive \underline{A}ctive \underline{A}ggregation and \underline{S}plitting (\Name) framework, falling within the cluster-adjustment scheme in active clustering. \Name{} features strategic active clustering adjustment on the initial cluster result, which is obtained by an adaptive clustering algorithm. In particular, our cluster adjustment is inspired by the quantitative analysis of Normalized mutual information gain under the information theory framework and can provably improve the clustering quality. The proposed \Name{} framework significantly elevates the performance and scalability of active clustering. 
In extensive experiments across diverse real-world datasets, \Name{} achieves desired results with significantly fewer human queries compared with existing methods. 
\end{abstract}

\input{introduction}
\input{method}
\input{experiment}

\input{rel_work}
\input{conclusion}

\section*{Acknowledgements}
This work is supported by the National Key Research and Development Program of China (2022YFB3104701) and Alibaba Group through Alibaba Research Intern Program. We appreciate the reviewers for their insightful feedback and advice, these constructive criticism and recommendations have been invaluable in helping us improve the quality of this work.

\section*{Impact Statement}

This paper presents a work that aims to advance the field of Active Clustering. There are some potential societal consequences 
of our work, none of which we feel must be specifically highlighted here.

\nocite{langley00}

\bibliography{example_paper}
\bibliographystyle{icml2024}

\newpage
\appendix
\onecolumn
\input{appendix}

\end{document}

%% file: introduction.tex
\section{Introduction}
\label{sec:intro}

In the realm of data science, clustering algorithms have emerged as a cornerstone technology within the domain of unsupervised learning~\cite{khanum2015survey,celebi2016unsupervised}. By automatically grouping similar data objects based on inherent structures and patterns within datasets, clustering provides an efficacious means to condense and structure complex information and is widely applied in image classification~\cite{caron2020unsupervised}, social network analysis~\cite{he2022sssnet}, etc. However, conventional clustering techniques often rely on static parameter settings and one-off computations, rendering them less adaptable to strange or expanding data environments (See Section \ref{methd:discu}). This context highlights the growing importance of human-computer collaborative active clustering approaches.

Active clustering refers to a paradigm that actively selects side information~\citep{Anand2014}, in the form of pairwise constraints, to maximally improve the clustering performance. Extensive work~\citep{gonzalez2023semi} has explored the combination of the strategic selection of pairwise constraints and semi-supervised clustering (SSC)~\citep{Basu_Banerjee_Mooney_2002}, and attains much lower query complexity compared to its semi-supervised counterpart~\citep{bilenko2004integrating}. 
Despite the practicality, current active clustering methods often suffer from high computational and query costs when the number of classes is large. 

SSC-based active clustering methods primarily assess the uncertainty of all pairwise constraints, and iteratively choose the most uncertain ones for expert queries. This manner, while systematic, faces notable challenges: 
it potentially relies on the assumption that a small set of initial pairwise constraints will rapidly cover most real classes. This assumption becomes increasingly unreliable in scenarios with large sample numbers 
$N$ and class numbers $K$. Therefore, some active clustering methods~\citep{van2018cobra,shi2020fast} shift from SSC to a cluster-adjustment scheme. This scheme involves over-clustering data into $k$ clusters via a specific clustering method (where $k$ is greater than $K$), and subsequently aggregates the resulting small clusters into larger ones based on pairwise constraints. However, it requires a proper cluster number $k$ as an input parameter, which is hard to determine in real applications. Moreover, the human query may be misleading when the selected samples are outliers, as they do not represent the majority sample of a cluster.

\looseness=-1 This work aims to overcome the drawbacks of existing cluster adjustment schemes. We first present a theoretical result that identifies conditions where aggregating two clusters does not reduce the normalized mutual information (NMI) between the resulting clustering and the real clustering. Here, NMI measures the overlap between two clustering results, with larger values indicating better performance (see Definition~\ref{def:nmi}). In addition, to guide active human queries, we quantify the impact of merging two clusters on the expected NMI value difference from the theoretical side. This characterization enables us to actively select cluster pairs that maximize the NMI gain.   


\looseness=-1 Building upon these theoretical results, we propose Adaptive Active Aggregation and Splitting (\Name), a generic cluster adjustment framework for active clustering. \Name{} operates in two stages: the \emph{adaptive initialization stage} and the \emph{active aggregation and splitting stage}. The first stage autonomously identifies an appropriate cluster number and produces initial clustering results using a designated clustering algorithm (e.g., K-means and hierarchical clustering). In the latter stage, \Name{} proactively identifies the cluster pair that is expected to enhance the NMI value mostly. Afterward, it assesses whether the majority of samples in a cluster are of the same class (i.e., cluster purity in Definition~\ref{def:purity}), then merges pure clusters queried to be in the same class by oracles, and divides impure clusters into pure subclusters and outliers. This stage will be repeated a few times to ensure convergence, where no more aggregation and splitting will happen. 

Our contributions are summarized as follows:
\vspace{-10pt}
\begin{itemize}[leftmargin=*] 



    \item We introduce A3S, a general active clustering algorithm that improves the quality of clustering by optimizing the NMI value. A3S strategically selects cluster pairs for aggregation, aiming to maximize the expected NMI improvement within a limited query cost. The quantification of NMI gain is guided by our information-theoretical analysis (Theorem 2.4). Additionally, A3S implements precise splitting on clusters that fail purity tests, ensuring the correction of outlier samples.

    \item Regarding implementation superiority, A3S can reveal the underlying ground truth clustering structure with substantially fewer human queries compared to traditional methods, especially when prior dataset knowledge is unknown. In addition, A3S locally adjusts the cluster labels, offering a more efficient alternative to rerunning semi-supervised clustering algorithms, and remains feasible for large datasets.

    \item In terms of practical performance, by conducting comprehensive evaluations on various real-world datasets, we demonstrate that A3S consistently surpasses all baseline models. Notably, baseline methods typically require more than 4000 queries to achieve high-quality clustering results for datasets containing thousands of samples. In contrast, A3S achieves this performance with only a few hundred queries.

\end{itemize}

%% file: method.tex
\section{Methodology}
\label{sec:method}

\begin{figure*}[t]
  \centering
  \includegraphics[width=\textwidth]{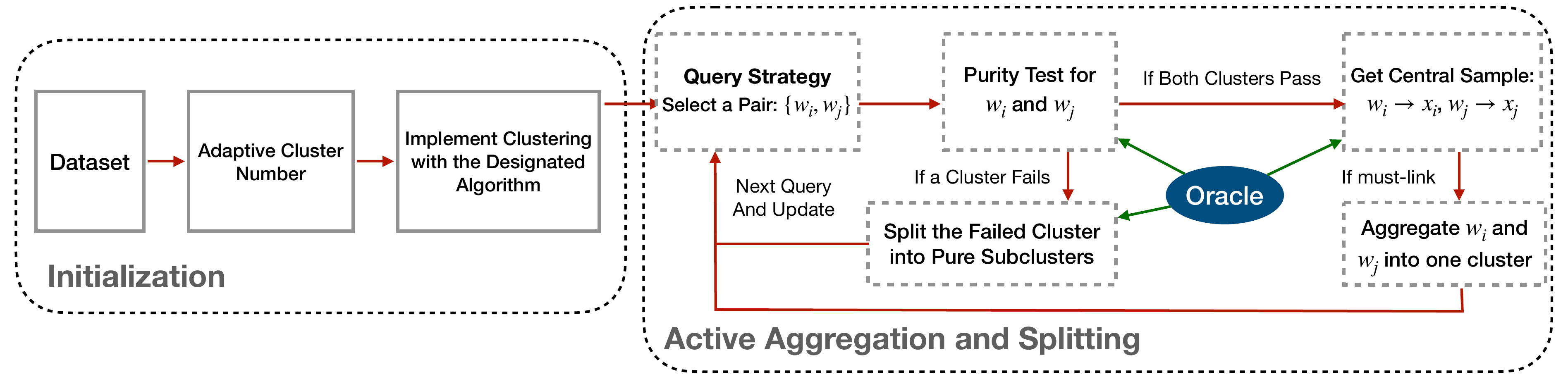}
  \vspace{-15pt}
  \caption{The workflow of \Name{} which consists of the adaptive clustering stage and active aggregation and splitting stage.}
  \label{fig:workflow}
  \vspace{-5pt}
\end{figure*}

\subsection{Notation and Definition}
\begin{definition}[Active Clustering]\label{def:ac}
    We denote the true classes of $N$ samples $X = \{x_1, \cdots, x_N\}$ by $Y=\{y_1, \cdots, y_N\}$, where $y_i\in \{1,\cdots,K\}$ and $K$ is the number of classes. Let the ground truth clustering of $X$ as $C=\{c_1,\cdots,c_K\}$, and the initial clustering result as $\Omega=\{w_1,\cdots,w_k\}$, where $X=\cup_{i=1}^{K}c_i=\cup_{i=1}^{k}w_i$ and $c_i\cap c_j=\emptyset, w_i\cap w_j=\emptyset, \forall i,j$.
    Active Clustering strategically selects sample pairs $(x_i,x_j)$, and requires the oracle to judge if $y_i=y_j$ (two samples are must-linked) or $y_i\neq y_j$ (two samples are cannot-linked). It then updates the clustering result $\Omega$ with the queried pairwise constraints accordingly. Active Clustering aims to utilize the queried constraints to maximally reduce the difference between $C$ and $\Omega$, which is measured by the normalized mutual information (NMI)~\citep{kvalseth1987entropy,vinh2009information}.
\end{definition}

\begin{definition}[Cluster Adjustment Scheme]\label{def:sche}
    We define a cluster adjustment scheme as a label update strategy employed by active clustering algorithms. Specifically, it entails the algorithm's process of either locally aggregating small clusters into larger ones or splitting impure clusters into several subclusters, guided by pairwise constraints.
\end{definition}

\begin{definition}[NMI]\label{def:nmi}
    NMI is a measure of how much common information two clustering results share. Given $N$ samples and their two clusterings $\Omega = \{w_1, \cdots, w_k\}$ and $\Omega' = \{w'_1, \cdots, w'_{k}\}$, we define the NMI value as 
    \small
    \$
    n = \frac{2 \mathbb{I}(\Omega; \Omega')}{\HH(\Omega) + \HH(\Omega')} = \frac{2 \mathbb{I}(\zeta; \zeta')}{\HH(\zeta) + \HH(\zeta')} ,
    \$
    \normalsize
    where $\zeta = (|w_1|/N, \cdots, |w_k|/N)$ and $\zeta' = (|w'_1|/N, \cdots, |w'_{k'}|/N)$ are two distributions induced by $\Omega$ and $\Omega'$, respectively. Here, $\mathbb{I}(\cdot;\cdot)$ denotes mutual information, and $\HH(\cdot)$ denotes entropy.
\end{definition}

We also introduce the notion of clustering purity~\citep{gonzalez2023semi}.

\begin{definition}[Purity]\label{def:purity}
    We define the \emph{dominant class} of a cluster $w_i$ as $\arg \max_j |w_i\cap c_j|$. We label the sample in a cluster that does not belong to its dominant class or the sample that is a single cluster itself as an outlier. Then, the purity of $w_i$ is $\max_j \frac{|w_i\cap c_j|}{|w_i|}$, and the purity of the clustering result $\Omega$ is quantified by $\frac{\sum_i\max_j|w_i\cap c_j|}{N}$.
\end{definition}

\subsection{Theoretical Analysis}
Active clustering that adopts the cluster-adjustment scheme does not rely on semi-supervised clustering for updating clustering results. Instead, it emphasizes the use of must-link and cannot-link constraints as indicators to merge or split clusters~\citep{van2018cobras}. This greedy strategy may, however, jeopardize the quality of clustering outcomes for several reasons. For example, the queried samples could be outliers, or the purity of a cluster might not be sufficiently high. In light of this, we introduce a pivotal theorem that offers clear guidance for aggregation actions. This is achieved through an evaluation of the NMI.

\begin{theorem}[Guarantee for Cluster Aggregation] \label{lemma:purity}
Denote the clustering of $N$ samples as $\Omega$, the ground truth clustering as $C$, and the NMI value of $\Omega$ with respect to $C$ as $n_1$. For any two clusters in $\Omega$, say $w_1$ and $w_2$, suppose they have a common dominant class $c_1$ with purities $t_1$ and $t_2$ respectively, where $t_1, t_2 \in [p,1]$. By aggregating $w_1$ and $w_2$ into a new cluster $w_{1,2}=w_1\cup w_2$, we arrive at a new clustering $\Omega^{\star}=\{w_{1,2}, w_3, \cdots,w_k\}$ with NMI value of $n_2$. This aggregation positively impacts clustering performance (i.e., $n_2 \ge n_1$) if $p\ge 0.7$ and $n_1 \ge 2 \cdot (1.0586-\min\{t_1, t_2\})$.
\end{theorem}

Theorem~\ref{lemma:purity} delineates the conditions for non-deteriorating cluster aggregation, and the detailed proof is in Appendix~\ref{appendix:pf:purity}. Specifically, merging two clusters can achieve provable benefits when their purity is at least $0.7$ and preceding NMI exceeds $2 \cdot (1.0586-\min\{t_1, t_2\})$. This finding is important as it suggests that the purity requirement for cluster aggregation is relatively mild, allowing for the inclusion of a small number of outliers within each cluster without compromising the overall clustering performance. 

We take a step further by formulating the expected gain in NMI value when we decide to query a sample pair from a cluster pair. This involves aggregating clusters if the query result is ``must-link" and keeping them separated if the result is ``cannot-link". 

\begin{definition}[Expected NMI Gain]\label{propo:exp_nmi}
    Suppose the dominant class of clusters $w_i$ and $w_j$ is $c_m$ and $c_n$ respectively, which remain unknown before querying. Let $n_1$ and $n_2$ represent the NMI values of the clustering result before and after aggregating the two clusters. We denote $\PP(c_m=c_n)$ as the probability that the oracle observes a `must-link' result. Then, we define the expected NMI gain from this query as follows:
    \# \label{eq:delta:nmi}
    \mathbb{E}[\Delta \texttt{NMI} \mid w_i, w_j] = \PP(c_m=c_n) \cdot (n_2-n_1),
    \#  
    \vspace{-25pt}
\end{definition}
In what follows, we present how to estimate $\PP(c_m=c_n)$ and $(n_2-n_1)$.

We use $e_{st} = 1/0$ to signify if $y_s$ equals $y_t$ or not, and denote the posterior pairwise probability as $\PP(e_{st}=1)$. We estimate the pairwise probability following previous probability clustering methods~\citep[e.g.,][]{liu2022mpc}, and discuss the estimation details in Appendix~\ref{appendix:prob-esti}. Then we can express the aggregation probability $\PP(c_m=c_n)$ as follows:
\small
\begin{equation}\label{prob:ori}
    \PP(c_m=c_n)= \frac{\prod\limits_{s\in w_i,t\in w_j} \PP(e_{st}=1)}{\prod\limits_{s\in w_i,t\in w_j} \PP(e_{st}=1)+\prod\limits_{s\in w_i,t\in w_j} \PP(e_{st}=0)}, 
\end{equation}
\normalsize
The detailed derivation of Eq. \eqref{prob:ori} is in Appendix~\ref{appendix:pf:merge_prob}. 


Moving forward, we focus on formulating $n_2 - n_1$. 
In line with the notations used in Theorem~\ref{lemma:purity}, we define $\Delta h = \HH(\Omega) - \HH(\Omega^*)$ and proceed with the following result:
\#
n_2 = \frac{2\II(\Omega^{\star};C)}{\HH(\Omega^{\star})+\HH(C)} \approx \frac{2\II(\Omega;C)}{\HH(\Omega)+\HH(C)-\Delta h},\notag
\#
where we use the fact that $\II(\Omega^*; C) \approx \II(\Omega, C)$ when the purity of $w_i$ and $w_j$ is sufficiently large. 
Moreover, when the sizes of clusters $w_i$ and $w_j$ are significantly smaller than the sample size $N$, the direct calculation gives that $\Delta h \ll \HH(\Omega)+\HH(C)$ (refer to Appendix~\ref{appendix:coro} for verification). Hence, we have 
\#
n_2 -n_1  \label{del_nmi}
\approx \frac{2\II(\Omega;C)\Delta h}{(\HH(\Omega)+\HH(C))^2} \propto \Delta h,
\#
Combing Eq. \eqref{eq:delta:nmi}, Eq. \eqref{prob:ori}, and Eq. \eqref{del_nmi}, we obtain our estimators of $\mathbb{E}[\Delta \texttt{NMI}\mid w_i, w_j]$ as follows:
\begin{equation}\label{delta_nmi}
    \mathbb{E}[\Delta \texttt{NMI}\mid w_i, w_j] \propto \text{RHS of Eq.~\eqref{prob:ori}} \cdot \Delta h.
\end{equation}
where RHS denotes the right-hand side.

\subsection{Adaptive Active Aggregation and Splitting}\label{sec:algo_detail}

Building on the analysis of cluster aggregation and expected query impact, we detail our Adaptive Active Aggregation and Splitting framework, which comprises two distinct stages. First, the Adaptive Clustering stage is introduced in Section~\ref{sec:bmaas:adapt}, where we describe the generation of initial clustering results. Second, in Section~\ref{sec:bmaas:active}, we discuss selecting pairwise constraints and updating clustering during the Active Aggregation and Splitting stage. The \Name{} workflow is illustrated in Figure~\ref{fig:workflow}, and the steps are summarized in Algorithm~\ref{algo:apc}.

\begin{figure*}[t]
  \centering
  \includegraphics[width=\textwidth]{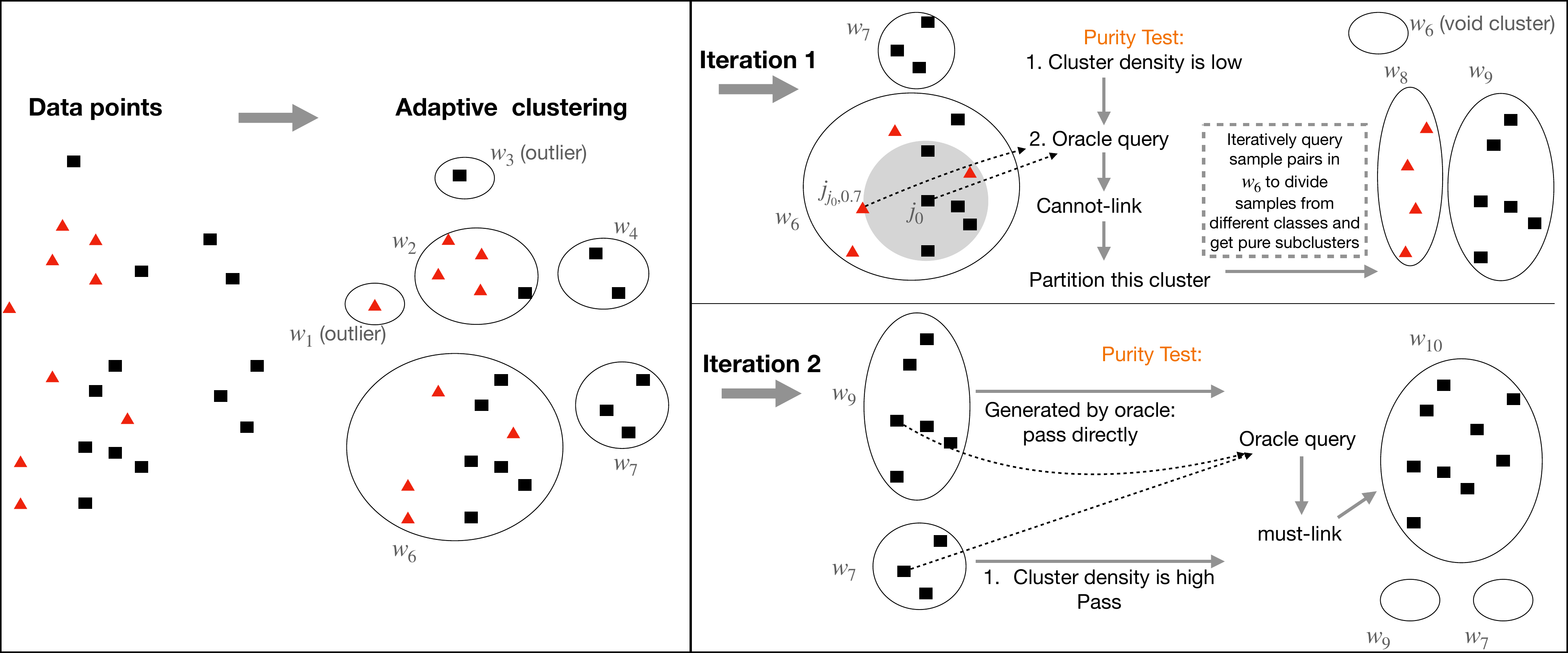}
  \vspace{-0.7cm}
  \caption{Case study for \Name{}. In iteration 1, the cluster $w_6$ does not pass the purity test and the oracle invests 12 queries to split it into two pure subclusters. In iteration 2, the query result for the two central samples is must-link and they are merged into one cluster.}
  \label{fig:case}
  \vspace{-5pt}
\end{figure*}

\subsubsection{Initialization via Adaptive Clustering}
\label{sec:bmaas:adapt}
Current active clustering methods often necessitate manually setting initial cluster numbers, a challenging task when the number of classes is unknown. In response, we propose using adaptive clustering methods to determine an appropriate cluster number for initialization. An adaptive clustering method organizes data based on local density, revealing the dataset's inherent structure. It handles noise by isolating each noisy sample into a separate cluster (i.e., outlier), avoiding unsuitable data grouping. This approach ensures a more natural and purer clustering outcome. Classic adaptive clustering methods include Probabilistic Clustering~\cite{liu2022mpc}, density-based clustering~\cite{zhang2021adaptive,khan2018adbscan}, among others. Additionally, the quality of adaptive clustering can be significantly enhanced in multi-view clustering scenarios~\cite{ljjslsmpc}. In this process, we aim to obtain a suitable reference for the number of clusters (typically larger than real class numbers), rather than the optimal clustering, hence do not require a precise hyper-parameter search. Once the adaptive cluster number is established, we can employ the desired clustering algorithm to produce the initial clustering result.

\subsubsection{Active Aggregation and Splitting}
\label{sec:bmaas:active}

\textbf{Query Strategy.}
To ensure a high success rate in establishing `must-link' connections among selected cluster pairs, We employ a two-step query strategy. The first step involves filtering out low-quality cluster pairs, focusing on those with aggregation probabilities at the top. In the second step, we utilize Eq. \eqref{delta_nmi} to calculate the expected NMI gain for these cluster pairs. Then we choose the pair with the highest NMI gain. By Theorem~\ref{lemma:purity}, it is necessary to evaluate the purity of the two chosen clusters and select one representative sample (i.e., it belongs to the dominant class of this cluster) from each cluster to form a sample pair. This pair will then be subjected to queries by oracles. To facilitate this process, we specially designed a purity test.


\vspace{-3pt}
\textbf{Purity Test.} 
For convenience, we employ the sphere structure to depict a cluster, where the centroid sample of a cluster is denoted as $j_0$. The other samples in the cluster are marked as $j_{i,\rho}$,  indicating that a sphere centered on sample $i$ with a radius of $d(i, j_{i,\rho})$ includes $\rho$ percent of the samples in the cluster. 
Considering that outlier samples typically reside in the outer regions of a cluster, and samples within impure clusters tend to be more sparsely distributed, we bifurcate the task of purity testing into two consecutive judgments. 

First, we evaluate how densely the samples are concentrated within a cluster. This assessment is formalized as the density test for a cluster $w$, expressed in Eq. \eqref{eq:dt}, where $\tau$ represents a pre-set threshold that determines the level of strictness in this density test, and $\mathbf{1}(\cdot)$ denotes the indicator function.
\#\label{eq:dt}
\mathcal{DT}(w)=\mathbf{1}\Big(\frac{\sum_{i\in w, j\in w(i)} \PP(e_{ij}=1)}{\sum_{i \in w} |w(i)|} > \tau\Big), \\
w(i)=\{j\mid j\in w, \PP(e_{ij}=1)<\PP(e_{ij_{i,\text{0.5}}}=1)\}\notag
\#
If a cluster fails in the density test, we select a sample pair as $(j_0,j_{j_0,\mathrm{0.7}})$, and require oracles to judge $\mathcal{P}(w)=\mathbf{1}(y_{j_0}=y_{j_{j_0,\mathrm{0.7}}})$. It estimates whether the purity of this cluster is higher than 0.7 (i.e., satisfying the conditions in Theorem~\ref{lemma:purity}). Overall, the purity test is as follows:
\#\label{purity_test}
\mathcal{PT}(w)= \mathcal{DT}(w) \quad \text{if} \quad \mathcal{DT}(w)\quad \text{else}\quad \mathcal{P}(w).
\#

\textbf{Clustering Update.}
In response to different outcomes in the purity test, we adopt the following strategies to update the clustering: (i) The purity test yields a result of 1 for both clusters, and we require the oracle to query their central samples. Should the query yield a must-link result, we will merge the clusters; if not, we will retain them as separate clusters. (2) The purity test yields 0 for at least one cluster, indicating a need for refinement, we proceed to split it into multiple subclusters with enhanced purity. Given the proven effectiveness of sample-based active clustering in managing small-scale datasets, as noted by ~\citet{basu2004active}, we apply this approach for the splitting task, and the detail is described in Algorithm~\ref{algo:pf}.

\begin{figure*}[!t]\vspace*{-18pt}
\begin{minipage}{0.5\textwidth}
\vspace*{0pt}
\begin{algorithm}[H]\small
\small
\caption{Adaptive Active Aggregation and Splitting}  
\label{algo:apc}
     \textbf{Input:} Data $X$, query limit $Q_{\max}$, index $q=0$, clustering algorithm~$\mathcal{A}$ \\
     \textbf{Initialization:} Using $\mathcal{A}$ to generate initial clustering $\Omega$ \\ 
    \For{iter in $1:L_2$}{
         Get a batch of $k$ candidate cluster pairs whose aggregation probability ranks top-$k$ as $\mathcal{C}=\{(w_i,w_j)| w_i,w_j \in \Omega\}$ \\
         Use Eq.~\eqref{delta_nmi} to measure cluster pairs in $\mathcal{C}$ and choose the top pair\\
        \eIf{$q< Q_{\max}$}{
            Implement Purity Test on $w_1$ and $w_2$ with Eq. \eqref{purity_test} \\ \label{oral_1}
            \eIf{Both clusters pass the test}{
                Select their central samples as $x_1$ and $x_2$\\
                 Require oracle to query $(x_1,x_2)$ \label{oral_2} \\
                Aggregate $w_1$ and $w_2$ if $(x_1,x_2)$ is must-linked 
            }{
                 Split $w_1$ or/and $w_2$ with Algorithm~\ref{algo:pf} \label{oral_3} 
            }
            Update constraints set with Algorithm~\ref{algo:fti}
            
            Update $q$ by adding the newly invested number of queries 
        }{
        Terminate and return the result
        }
    }
\end{algorithm}
\end{minipage}\hspace*{0.01\textwidth}
\begin{minipage}{0.5\textwidth}\vspace*{0pt}
\begin{algorithm}[H]\small
\caption{Subcluster Partition}  
\label{algo:pf}
     \textbf{Input:} Cluster $w$; subcluster lists $\mathcal{N}=\{\}$ \\
     Sort samples in $w$ in ascending order by their distance to the centroid\\
    \For{i in $w$}{
         Select one sample from each subcluster in $\mathcal{N}$, and get $S$\\
         Query $i$ with sample in $S$ till a must-link is reached or all samples in $S$ have been queried\\
        \eIf{$i$ is must-linked to $j$ in $S$}{
            Move $i$ from $w$ to the subcluster $\mathcal{N}_u$, $j\in\mathcal{N}_u$ 
        }{
            Move $i$ from $w$ to an empty subcluster, and add the new subcluster $\mathcal{N}_i=\{i\}$ to $\mathcal{N}$ 
        }   
        
    }
\end{algorithm}
\vspace{-8mm}
\begin{algorithm}[H]\small
	\caption{Fast Transitive Inference}  
	\label{algo:fti}
		\textbf{Input:} State matrix $S$, new constraints $(s,t)$. \\
		\For{$i$ in $(s,t)$}{
            Get $\mathcal{ML}=\{j|S[i,j]=1\}$ \\
            Get $\mathcal{CL}=\{j|S[i,j]=-1\}$ \\
            Let $S[p,q]=1$, for $p,q\in \mathcal{ML}$ \\
            Let $S[p,q]=-1$, for $p\in\mathcal{ML},q\in\mathcal{CL}$ 
            }
\end{algorithm}
\end{minipage}
\vspace{-5pt}
\end{figure*}

\textbf{Transitive Inference.} The must-link and cannot-link constraints possess the transitivity property (e.g., $(x_1,x_2),(x_2,x_3)$ are must-linked, then $(x_1,x_3)$ is must-linked). To store the constraints, we define a state matrix as $S=\{s_{ij}\}_{N\times N}, s_{ij}\in\{-1,0,1\}$. Here, 1/-1 denotes must-link/cannot-link, and 0 indicates an unqueried state. To avoid unnecessary queries, we need to augment the constraints set each time a new constraint is added. We assert that this expansion is only relevant to the preceding constraints that share a common sample with the new constraint, and propose a Fast Transitive Inference (Algorithm~\ref{algo:fti}) method to update the constraints. The correctness of this assertion and algorithm is proved in Appendix~\ref{appendix:exp:fti}. 


\subsubsection{Discussion.}
\label{methd:discu}
\textbf{Complexity Analysis.}
The computational complexity of the \Name{} algorithm comprises three distinct components. During the pairwise probability estimation, the complexity is $O(NM)$, where $M$ denotes the number of neighbors considered for each sample. This is due to the need to only account for the neighboring clusters in each query. In the adaptive clustering process, the complexity depends on the specific algorithm, which we mark as $O(A)$. The complexity of the query strategy is less than $O(k^2+kL_2)$, where $k$ is the initial cluster number and $L_2$ denotes the iteration number for \Name{}. Because there are $k^2$ cluster pairs at the beginning, and we only need to re-calculate $k$ cluster pairs in each iteration. Consequently, the computational complexity of \Name{} is at most $O(NM+A+k^2+kL_2)$.

\textbf{Application of \Name{}.}
\Name{} excels in two key scenarios: First, in clustering tasks lacking prior data information (strange environment), like the number of classes or sample distribution, where traditional methods falter or require extensive human queries. Second, in ongoing real-world applications needing regular data aggregation (expanding environment), such as weekly updates. Here, \Name{} adeptly merges new with existing clusters, efficiently managing redundancy. It's particularly suitable when each period's data is already high-quality, naturally meeting purity constraints.

%% file: experiment.tex
\vspace{-0.1cm}
\section{Experiments}
\label{sec:experiment}
\vspace{-0.1cm}
 
We organize the experiments as follows: we explain the experimental setup in Section~\ref{exp:setting}; we compare \Name{} with state-of-the-art active clustering methods and present the detailed results in Section~\ref{exp:rq1}; then we compare the performance of \Name{} when applied to different clustering algorithms in Section~\ref{exp:rq2}; lastly, we explore the influence of components in \Name{} in Section~\ref{exp:rq3} to ~\ref{exp:rq5}.

\subsection{Experimental Setup}
\label{exp:setting}

\textbf{Datasets.} We sampled six datasets from four real-world image sources for the experiments: \texttt{Market-1501}~\citep{zheng2015scalable}, which comprises human body images from 1501 individuals. We use two subsets: \texttt{MK20} (351 images from 20 people) and \texttt{MK100} (1650 images from 100 people); \texttt{Humbi}~\citep{yu2020humbi}, a large multiview image dataset focused on human expressions like faces, and we extracted a subset \texttt{Humbi-Face} containing 5600 face images from 100 different people; \texttt{Handwritten}~\citep{dua2017uci}, a collection containing 2000 samples of handwritten digits from `0' to `9'. We use the Fourier coefficient features in the experiments. (4) \texttt{MS1M}~\citep{guo2016ms}, a substantial benchmark dataset commonly used in face recognition tasks, and we sampled two large subsets: \texttt{MS1M-10k}, \texttt{MS1M-100k}. The details of these datasets are shown in Table~\ref{tab:data_detail}. Regarding the commonly used benchmarks in previous constrained clustering methods (e.g., UCI datasets \citep{asuncion2007uci}), they are not appropriate for our problems due to the very small sample and class sizes. Consequently, we have not evaluated the performance on those benchmarks. We seek to demonstrate that our method is generally workable for different types of data/applications with a wide range of cluster numbers and sample numbers. 

\begin{figure*}[t]
	\centering
        \subfigure{
        \begin{minipage}[t]{0.24\linewidth}
            \centering
            \includegraphics[width=1\textwidth]{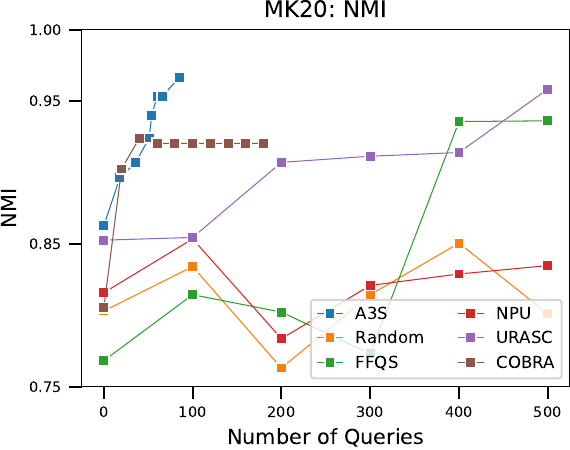}\\
        \end{minipage}%
        }
        \subfigure{
        \begin{minipage}[t]{0.24\linewidth}
            \centering
            \includegraphics[width=1\textwidth]{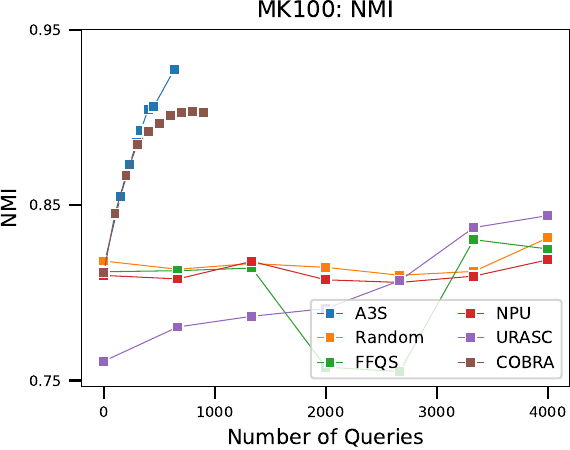}\\
        \end{minipage}%
        }
        \subfigure{
        \begin{minipage}[t]{0.24\linewidth}
            \centering
            \includegraphics[width=1\textwidth]{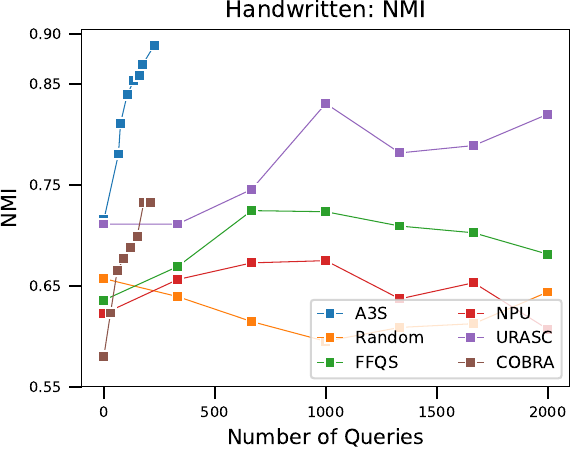}\\
        \end{minipage}%
        }
        \subfigure{
        \begin{minipage}[t]{0.24\linewidth}
            \centering
            \includegraphics[width=1\textwidth]{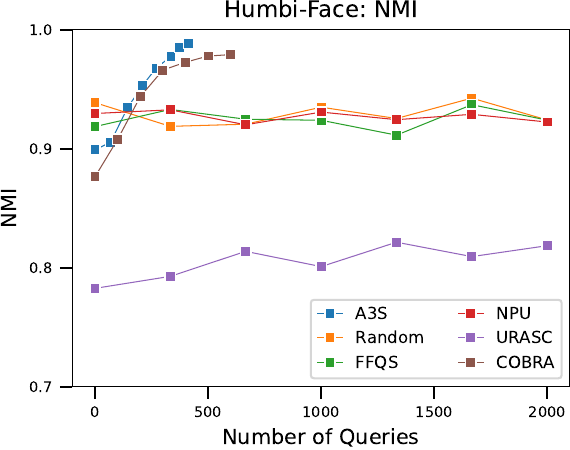}\\
        \end{minipage}%
        }
        \subfigure{
        \begin{minipage}[t]{0.24\linewidth}
            \centering
            \includegraphics[width=1\textwidth]{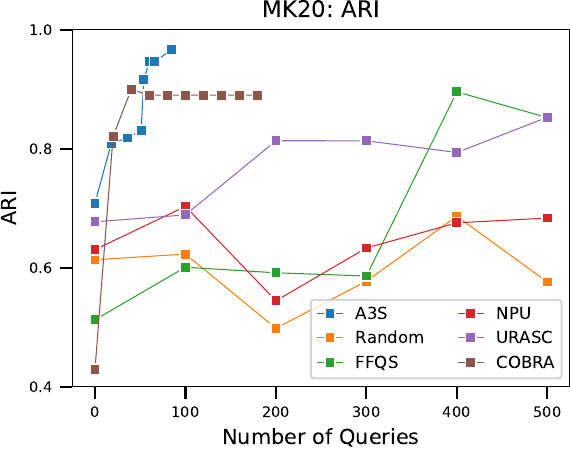}\\
        \end{minipage}%
        }
        \subfigure{
        \begin{minipage}[t]{0.24\linewidth}
            \centering
            \includegraphics[width=1\textwidth]{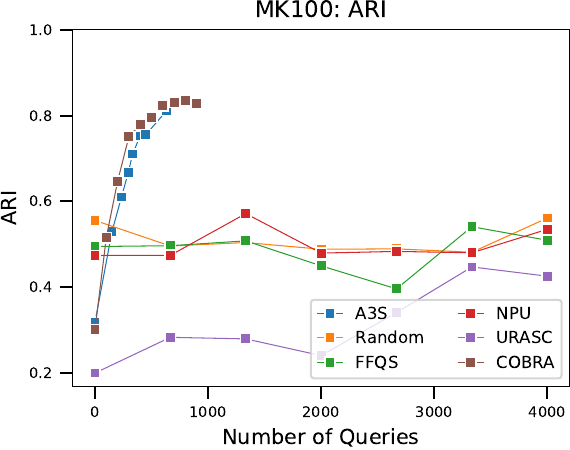}\\
        \end{minipage}%
        }
        \subfigure{
        \begin{minipage}[t]{0.24\linewidth}
            \centering
            \includegraphics[width=1\textwidth]{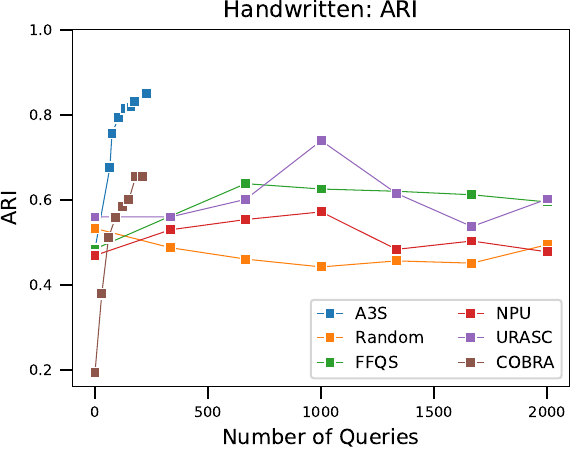}\\
        \end{minipage}%
        }
        \subfigure{
        \begin{minipage}[t]{0.24\linewidth}
            \centering
            \includegraphics[width=1\textwidth]{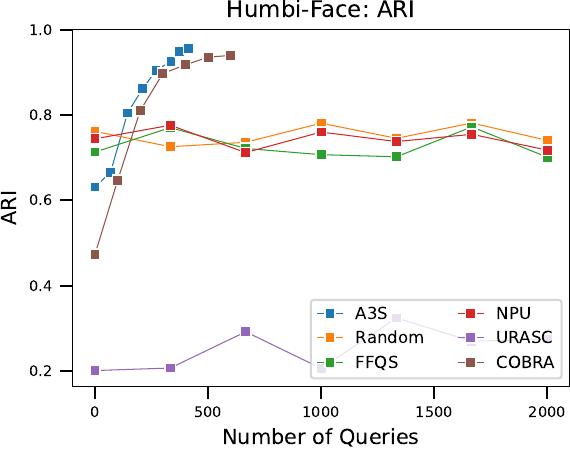}\\
        \end{minipage}%
        }
        \vspace{-10pt}
	\caption{Comparing the performance of \Name{} and baselines on four datasets in terms of query count. More queries are invested for baselines to illustrate their characteristics.
	}
	\label{fig:main curve}
	\vspace{-10pt}
\end{figure*}

\textbf{Baselines.} 
To validate the performance of our methods, a set of baselines and state-of-the-art algorithms are compared. \textbf{Random}~\citep{basu2003comparing} randomly selects pairwise constraints. \textbf{FFQS}~\citep{basu2004active} uses the farthest-first scheme to acquire diversified samples and pairwise constraints. \textbf{NPU}~\citep{xiong2013active} uses the classic entropy-based principle to select informative samples to construct pairwise constraints. We use PCKMeans~\citep{basu2003comparing} as the semi-supervised clustering algorithm for these three methods, because PCKMeans best suits the pairwise constraints manner, and is suitable for large data sets with sparse high-dimensional data~\citep{cai2023review}. \textbf{URASC}~\citep{xiong2016active} aims to iteratively query pairwise constraints that can maximally reduce the uncertainty of spectral clustering. 
\textbf{COBRA} over-cluster a dataset with K-means, then iteratively selects the closest cluster pairs for querying.

\begin{table}[t]
    \centering
    \setlength{\abovecaptionskip}{0.1cm}
    \setlength{\belowcaptionskip}{0cm}
    \caption{Statistical information about six datasets. $N$, $K$, and $D$ represent the sample numbers, class numbers, and the dimension of features. $b$ indicates whether the sample quantities between different classes are balanced.}
    \resizebox{0.46\textwidth}{!}{
    \begin{tabular}{c|cccccccc}
        \toprule
         & \texttt{MK20} & \texttt{MK100} & \texttt{Handwritten} & \texttt{Humbi-Face} & \texttt{MS1M-10k} & \texttt{MS1M-100k} \\
        \midrule
        $N$ & 351 & 1650 & 2000 & 5600 & 10000 & 100000 \\
        $K$ & 20 & 100 & 10 & 100 &  146 & 1469  \\
        $D$ & 256 & 256 & 76 & 256 & 512 & 512 \\
        $b$ & \XSolidBrush & \XSolidBrush & \Checkmark &\Checkmark  & \XSolidBrush& \XSolidBrush\\
        \bottomrule
    \end{tabular}}
    \label{tab:data_detail}
    \vspace{-10pt}
\end{table}

\textbf{Implementation.} For the baseline methods, we maintain the same hyperparameter settings as reported in their original papers to ensure fairness in the comparison. Note that these baseline methods require the real cluster number as input, which is provided in our experiments. In many clustering applications, however, this number is typically not known beforehand, thus they are at an advantage. For the initialization of \Name{}\footnote{The code is available at https://github.com/xiangtanshi/A3S.}, we use isotonic regression to learn pairwise probability and utilize Fast Probabilistic Clustering (FPC)~\cite{liu2022mpc} as the adaptive clustering method. Here, we choose FPC because it only requires the pairwise probability for automatic data grouping, which is easy to implement. 
\textbf{COBRA} cannot assign a proper initial cluster number itself, so we use the same $k$ as \Name{} for a fair comparison. More details are presented in Appendix~\ref{appendix:imple}.


\textbf{Evaluation.} As discussed in \citet{vinh2009information}, NMI can exhibit bias towards fine-grained clustering. Therefore, in addition to NMI, we employ the Adjusted Rand Index (ARI)~\citep{hubert1985comparing} to evaluate the performance. NMI and ARI fall within the range of (0,1] and [-1,1] respectively, with larger values indicating superior clustering performance. To further investigate whether \Name{} effectively resolves the category fission and recover the real clustering structure, we introduce two supplementary metrics: (1) the Fission Rate ($\Upsilon=\frac{k}{K}$), where $K$ is the real class number and $k$ is the resulting cluster number; and (2) the entropy ratio between resulting clusters ($\Omega$) and real class partitions ($C$), $r=\frac{\HH(\Omega)}{\HH(C)}$. When $\Upsilon$ and $r$ approaches 1, we conclude that \Name{} has effectively mitigated the category fission problem, and successfully recovered the true structure of $C$.

\subsection{Comparison with SOTA Active Clustering Methods}
\label{exp:rq1}

We compare \Name{} and five baseline methods on four datasets with different numbers of queries in terms of both NMI and ARI.  
The results are shown in Figure~\ref{fig:main curve}. Overall, \Name{} has higher NMI and ARI values than other methods on these data sets when setting the same number of queries, and \Name{} requires only a small amount of queries to improve the NMI and ARI values significantly. In addition, we observe that both \Name{} and \textbf{COBRA} (cluster-based methods) improve steadily with the increase of queries, while \textbf{Random}, \textbf{FFQS}, \textbf{NPU} and \textbf{URASC} (semi-supervised clustering based methods) show fluctuations on all data sets. This is because genuine supervisory information can sometimes be detrimental to clustering, as it may introduce violations (e.g., $(x_i, x_j)$ is cannot-linked, but their similarity to $x_k$ is both very high). This problem also exists for the latest semi-supervised clustering methods such as PCSKM~\citep{vouros2021semi}. However, it does not mean that this line of work is not applicable. One common advantage of them is that they can ultimately improve the NMI and ARI value to 1.0 if enough queries can be provided (typically less than $N\times\log(N)$). They are a good choice when the target is to reveal the cluster identity for all samples accurately, or only low-quality features are available and the NMI of the initial clustering is lower than 0.2.

The detailed running result of \Name{} is in Table~\ref{tab:main_result}. \Name{} significantly reduces the fission rate and the entropy ratio to almost 1.0 on all datasets, with a high clustering purity. This validates that \Name{} can effectively reveal the true clustering structure of data. It's important to highlight that these results are obtained without any prior knowledge about the number of classes or class distribution. 
Additionally, \Name{} exhibits robustness to the choice of adaptive cluster number. As detailed in Appendix~\ref{appendix:exp_result}, increasing the adaptive number does not significantly alter the number of queries needed to obtain the desired result.

Next, we delineate the distinctions in the results yielded by \Name{} and \textbf{COBRA}. Although \textbf{COBRA} quickly improves the NMI and ARI values, it cannot further enhance the clustering even when more queries are invested (all must-link clusters have already been discovered), leading to a performance ceiling dictated by cluster purities. Recent methods like COBRAS \citep{van2018cobras} and AQM+MEE \citep{deng2023query} also encounter similar issues. In contrast, \Name{} not only delivers high-quality clusters but also identifies a subset of outlier samples. This approach enables continued augmentation of the NMI and ARI values through strategic querying of these outliers in combination with existing clusters. For instance, \Name{} and \textbf{COBRA} reach NMI values of 0.93 and 0.90 on \texttt{MK100} separately, but the overall clustering purity of \Name{} is 0.9752, far surpassing \textbf{COBRA} that is 0.8509. Further, by querying the outlier samples of \Name{} with their neighbor clusters until `must-link' is observed and they are aggregated to the corresponding clusters, \Name{} can reach an NMI of 0.99 with less than 500 more queries.

\begin{table}[t] 
	\centering
	\caption{Detailed results of \Name{} in Figure~\ref{fig:main curve}: the total running time (seconds), initial and ending fission rate ($\Upsilon$) and entropy ratio ($r$), and the final clustering purity.}
	\resizebox{\columnwidth}{!}{
		\begin{tabular}{c|c|cc|ccc}
			\hline 
		 & Time & $\Upsilon_{\mathrm{init}}$ &  $\Upsilon_{\mathrm{end}}$ & $r_{\mathrm{init}}$ & $r_{\mathrm{end}}$ & purity\\ 
			\hline
			\texttt{MK20} & 0.93 & 2.05& 0.95 & 1.25 & 1.01 & 0.9886 \\
            \texttt{MK100}& 5.58& 2.83 & 0.94 & 1.30 & 1.08 & 0.9752  \\
            \texttt{Handwritten} & 6.39 & 8.6& 1.00 & 1.55 & 1.00 & 0.9165 \\
            \texttt{Humbi-Face} & 37.62 & 3.48 & 1.01 & 1.19& 1.0 & 0.9745 \\
            \hline
	\end{tabular}}
	\label{tab:main_result}
	\vspace{-10pt}
\end{table}

\subsection{\Name{} for Different Clustering Algorithms}
\label{exp:rq2}

\begin{figure}[t]
	\centering
        \subfigure{
        \begin{minipage}[t]{0.48\linewidth}
            \centering
            \includegraphics[width=1\textwidth]{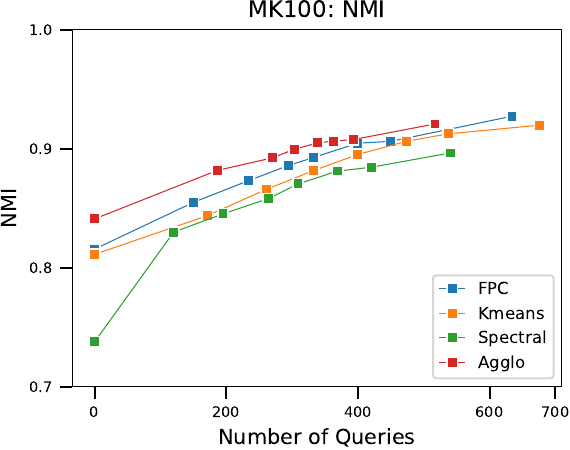}\\
        \end{minipage}%
        }
        \subfigure{
        \begin{minipage}[t]{0.48\linewidth}
            \centering
            \includegraphics[width=1\textwidth]{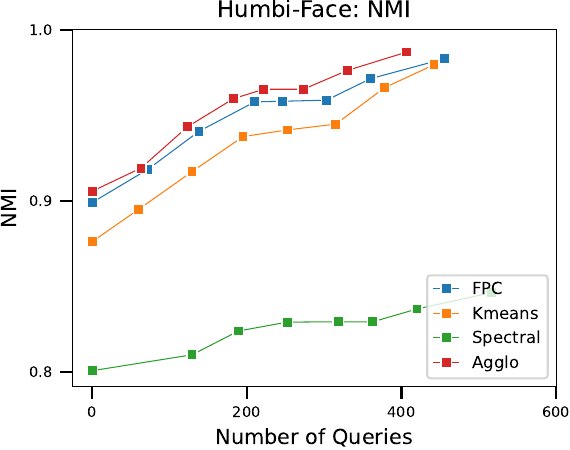}\\
        \end{minipage}%
        }
        \vspace{-20pt}
	\caption{Performance of \Name{} on \texttt{MK100} and \texttt{Humbi-Face} when utilizing different clustering algorithms to generate the initial clustering result.
	}
	\label{fig:abl_clus}
	\vspace{-10pt}
\end{figure}

\begin{table}[t] 
	\centering
	\caption{The ARI value between the clustering results of different clustering algorithms. We use F, K, S, and A to represent Fast Probabilistic Clustering, K-means Clustering, Spectral Clustering, and Agglomerative Clustering.}
	\resizebox{\columnwidth}{!}{
		\begin{tabular}{c|>{\columncolor{LightGreen}}c>{\columncolor{LightGreen}}c>{\columncolor{LightGreen}}c>{\columncolor{LightGreen}}c|>{\columncolor{LightBlue}}c>{\columncolor{LightBlue}}c>{\columncolor{LightBlue}}c>{\columncolor{LightBlue}}c}
			\hline 
            & \multicolumn{4}{c|}{\texttt{MK100}} & \multicolumn{4}{c}{\texttt{Humbi-Face}} \\
            
		  & \cellcolor{white} F & \cellcolor{white} K &\cellcolor{white} S &\cellcolor{white} A &\cellcolor{white}  F &\cellcolor{white} K &\cellcolor{white} S &\cellcolor{white} A \\
			\hline
		F & 1.000 & 0.581 & 0.187 &0.597 & 1.000 &         0.600 & 0.272 & 0.805 \\
            K &  0.581 & 1.000 & 0.140 & 0.408 &  0.600 & 1.000 & 0.178 &  0.573 \\
            S & 0.187 & 0.140 & 1.000 & 0.249 & 0.272 & 0.178 & 1.000 & 0.276 \\
            A & 0.597 & 0.408 & 0.249 & 1.000 & 0.805 & 0.573 & 0.276 & 1.000  \\
            \hline
	\end{tabular}}
	\label{tab:clus_corre}
	\vspace{-10pt}
\end{table}

\begin{figure}[t]
	\centering
        \subfigure{
        \begin{minipage}[t]{0.48\linewidth}
            \centering
            \includegraphics[width=1\textwidth]{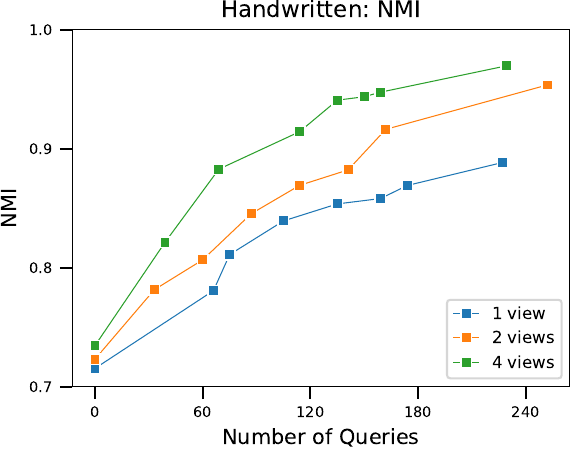}\\
        \end{minipage}%
        }
        \subfigure{
        \begin{minipage}[t]{0.48\linewidth}
            \centering
            \includegraphics[width=1\textwidth]{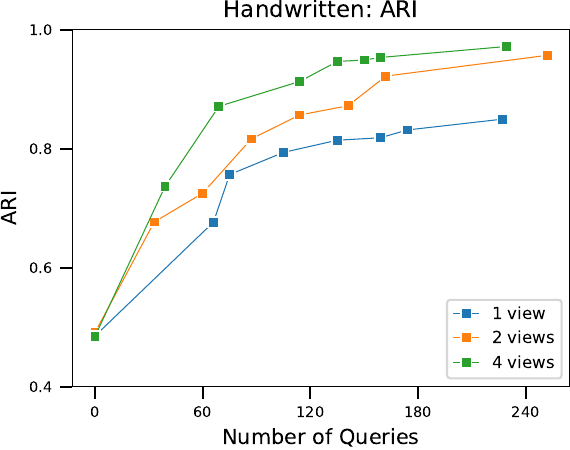}\\
        \end{minipage}%
        }
        \vspace{-15pt}
	\caption{Performance of \Name{} on \texttt{Handwritten} when 1, 2, or 4 views of the feature are used.
	}
	\label{fig:abl_view}
	\vspace{-5pt}
\end{figure}

To assess the compatibility of \Name{} on different clustering algorithms, we additionally use three classic clustering algorithms to generate the initial clustering (adaptive cluster number is provided by FPC): K-means clustering~\cite{choo2020k}, Spectral clustering~\citep{von2007tutorial} and Agglomerative Clustering~\cite{murtagh2014ward}. We test these versions of \Name{} on \texttt{MK100} and \texttt{Humbi-Face}, and the results are shown in Figure~\ref{fig:abl_clus}. Besides, we quantify the difference between these initial clustering results with their mutual ARI value in Table~\ref{tab:clus_corre}. We have two observations: first, the initial clustering outcomes derived from different algorithms exhibit substantial variability (mutual ARI value is typically lower than 0.6); second, \Name{} demonstrates a consistent ability to enhance the clustering performance across various algorithms. The results validate that \Name{} is robust to the initial clustering results, and can be effectively applied to other clustering algorithms without modification. In contrast, previous active clustering methods are typically designed and applicable to a specific clustering algorithm like DBSCAN~\cite{Mai_He_Hubig_Plant_Bohm_2013} and Spectral Clustering~\cite{xiong2016active}.

\subsection{Influence of Estimated Pairwise Probability Quality}\label{exp:rq3}
Better pairwise probability can lead to improved clustering performance, but its impact on \Name{} is unexplored. Multi-view clustering~\citep{yang2018multi} is the major approach in this domain, and we utilize \texttt{Handwritten} to investigate this aspect. The complete \texttt{Handwritten} dataset comprises four distinct feature types for each sample. Following the setup in \citet{liu2022mpc}, we initially learn the pairwise probabilities using features from each view. Subsequently, pairwise probabilities are aggregated across $V$ different views employing the following formula~\citep{liu2022mpc}: $\PP(e_{ij}=1|d_{ij}^1,\cdots,d_{ij}^V)=\frac{\prod_{m=1}^{V} \PP(e_{ij}=1|d_{ij}^m)}{\prod_{m=1}^{V} \PP(e_{ij}=1|d_{ij}^m)+\prod_{m=1}^{V} \PP(e_{ij}=0|d_{ij}^m)}$, where $d_{ij}^m$ is the pairwise distance in the $m$-th view. 

In the investigation of \Name{}'s performance with varying views (one/two/four) as illustrated in Figure~\ref{fig:abl_view}, we observed that employing multi-view clustering slightly improves the initial clustering performance, but significantly boosts \Name{}'s overall effectiveness. This approach leads to quicker convergence and nearly perfect NMI and ARI values. The improvement is largely attributed to multi-view aggregated pairwise probabilities, which make outlier samples distinguishable from neighbor samples that are from different classes, thereby preventing them from being in one cluster during initial clustering. Empirically, In four-view scenarios, despite a higher initial fission rate of 10.1, the initial clustering purity improves from 0.916 to 0.989.

\begin{figure}[t]
	\centering
        \subfigure{
        \begin{minipage}[t]{0.46\linewidth}
            \centering
            \includegraphics[width=1\textwidth]{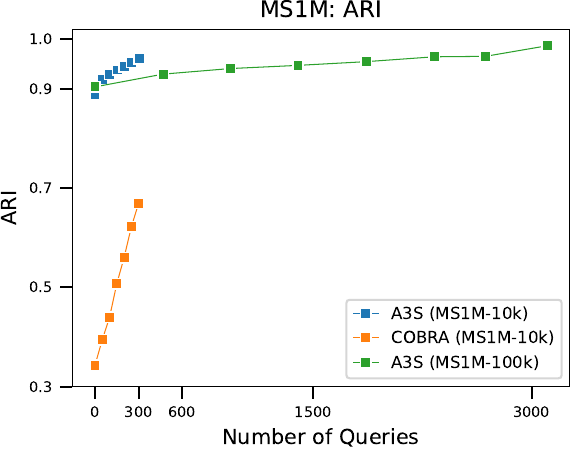}\\
        \end{minipage}%
        }
        \subfigure{
        \begin{minipage}[t]{0.50\linewidth}
            \centering
            \includegraphics[width=1\textwidth]{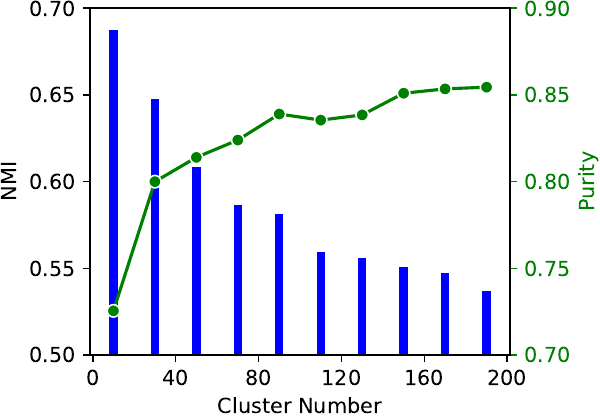}\\
        \end{minipage}%
        }
        \vspace{-15pt}
	\caption{Left: performance of \Name{} on \texttt{MS1M-10k} and \texttt{MS1M-100k}. Right: trade-off between initial clustering quality and initial clustering purity for \textbf{COBRA} on \texttt{Handwritten}.
	}
	\label{fig:abl_ms1m}
	\vspace{-10pt}
\end{figure}

\subsection{Performance of \Name{} on Large datasets.}\label{exp:rq5}
We further test \Name{} on two large datasets \texttt{MS1M-10k} and \texttt{MS1M-100k}, where the initial cluster number is determined by FPC for both \Name{} and \textbf{COBRA}. The result is in Figure~\ref{fig:abl_ms1m}. We have two observations: (1) \Name{} can reach near-optimal ARI value (0.995) with only 3100 queries for \texttt{MS1M-100k} in 8581.86 seconds, which validates that it is scalable to large datasets. As far as we know, this is the largest dataset ever used in active clustering research. (2) \textbf{COBRA} is sensitive to the initial cluster number and can waste too many queries to get a bad clustering result. 

\subsection{Case Study for \Name{}}\label{exp:rq4}
We present a case study in Figure~\ref{fig:case} to demonstrate \Name{}'s mechanism using a two-class dataset. Initially, data points are clustered into seven groups, including two outlier clusters. In the first iteration, clusters $w_6$ and $w_7$ are selected for their largest expected impact on the NMI value. During the purity test, $w_6$ fails the density test, leading to a subsequent query between its central sample ($j_0$) and the margin sample ($j_{j_0,0.7}$) of the sphere that is centered at $j_0$ and contains 70\% samples of $w_6$. The cannot-link result indicates $w_6$ fails the purity test, so it is split into two sub-clusters using Algorithm~\ref{algo:pf}, which costs 12 queries. In the second iteration, clusters $w_9$ and $w_7$ are selected, both passing the purity test. Their central samples are queried, and they are merged into $w_{10}$ following the must-link result. 

\looseness=-1 \Name{} addresses outlier detection through purity tests and subcluster partitioning. However, the purity test can discover most but not all low-quality clusters whose purity is lower than 0.7. As a remedy, we could resort to multi-view features which can significantly improve the clustering purity and reduce the number of outliers, as shown in Section~\ref{exp:rq3}. We remark that when multi-view data is available, the purity test and subcluster partition of \Name{} can be omitted, streamlining the procedure.
In contrast, \textbf{COBRA} attempts to mitigate outlier issues by opting for a larger cluster number to enhance initial clustering purity, which is inefficient as depicted in the right image of Figure~\ref{fig:abl_ms1m} (the clustering purity increases very slowly).


%% file: rel_work.tex
\section{Related Work}
\label{sec:rel}

\textbf{Query Strategy in Active Clustering.} Recent Active clustering methods have embraced a trend of incorporating sample uncertainty into their query strategies. These methods frequently utilize entropy to quantify uncertainty~\citep{abin2016clustering,xiong2016active,shi2020fast}. A common task involves estimating the probability of a sample belonging to different clusters or neighborhoods~\cite{xiong2013active}. Additionally, alternative criteria such as maximum expected error reduction~\citep{wang2010active} and maximum expected clustering change~\citep{biswas2014active} have been proposed to assess the stability of clustering results when perturbing the similarity values between two samples.

\textbf{Constraints in Semi-supervised Clustering.}    
When using the must-link and cannot-link constraints to perform SSC, two aspects are usually taken into consideration: the transitive inference of constraints and the combination of constraints to specific clustering algorithms. A few studies~\citep {lutz2021active} address transitive inference with graph-based techniques instead of a brute-force manner. In addition, recent studies that optimize clustering results with constraints in SSC have explored various approaches, ~\citet{vouros2021semi} explores Kmeans clustering for high dimensional data; ~\citet{yang2022semi} attempts to correctly infer the number of clusters for hierarchical clustering; \citet{ren2018semi} tries to utilize prior knowledge to determine a proper cluster number for density-based clustering; ~\citet{chen2022adaptive} develops a graph-based SSC that is robust to noise, but is not suitable for large datasets. However, there remains a relatively unexplored potential in integrating these modern SSC approaches with active clustering, which presents a promising avenue for future research.


%% file: conclusion.tex
\section{Conclusion}
\label{sec:con}
This paper studies the cluster-adjustment scheme in active clustering, offering theoretical guidance for non-deteriorating cluster aggregation and quantifying the impact of human queries and aggregation operations. We then propose \Name{}, a general framework that does not rely on the dataset prior. Through extensive testing, \Name{} demonstrates its effectiveness on diverse real-world datasets with varying class numbers and distributions. We will explore the application of \Name{} on more complex multi-view datasets and gigantic datasets at the million level in the future.

%% file: appendix.tex
\section{Proofs}
\label{sec:app-a}


\subsection{Proof of Theorem \ref{lemma:purity}} \label{appendix:pf:purity}
\begin{proof}[Proof of Theorem \ref{lemma:purity}]
    Let $p$ and $q$ denote the sizes of $w_1$ and $w_2$, respectively. We further assume that the class index of these outliers is $i_j \in \{1, 2, \cdots, K\}$, where $j\in \{1,2,\cdots,(1-t_1)p+(1-t_2)q\}$. For ease of presentation, for any class index $i \in \{1, 2, \cdots, K\}$, we use $s_i$ to denote the class size, i.e., $|c_i|=s_i$. Without loss generality, we assume that $q \ge p$ throughout this proof. 
    
    By the definition of mutual information, we have
    \begin{align}\label{eq:mi}
        \mathbb{I}(\Omega^{*}; C) &=\mathbb{I}(\Omega; C)+\sum_{\tau=1}^K \PP(w_{1,2}\cap c_\tau)\log\frac{\PP(w_{1,2}\cap c_\tau)}{\PP(w_{1,2})\PP(c_\tau)}-\sum_{\tau = 1}^K \PP(w_{1}\cap c_\tau)\log\frac{\PP(w_{1}\cap c_\tau)}{\PP(w_{1})\PP(c_\tau)} \notag \\
        & \qquad - \sum_{\tau = 1}^K \PP(w_{2}\cap c_\tau)\log\frac{\PP(w_{2}\cap c_\tau)}{\PP(w_{2})\PP(c_\tau)} \notag\\
        &=\mathbb{I}(\Omega; C) + \underbrace{\sum_{\tau = 1}^K \frac{|w_{1,2}\cap c_\tau|}{N}\log\frac{N\cdot|w_{1,2}\cap c_\tau|}{|w_{1,2}|\cdot|c_\tau|}}_{\displaystyle\mathrm{(I)}} - \underbrace{\sum_{\tau = 1}^K \frac{|w_{1}\cap c_\tau|}{N}\log\frac{N\cdot|w_{1}\cap c_\tau|}{|w_{1}|\cdot|c_\tau|} }_{\displaystyle\mathrm{(II)}} \notag\\
        & \qquad - \underbrace{\sum_{\tau = 1}^K \frac{|w_{2}\cap c_\tau|}{N}\log\frac{N\cdot|w_{2}\cap c_\tau|}{|w_{2}|\cdot|c_\tau|} }_{\displaystyle\mathrm{(III)}}.
    \end{align}

    Then we bound these three terms respectively. For Term (I), we have
    \# \label{eq:101}
    \mathrm{(I)} &= \frac{|w_{1,2}\cap c_1|}{N}\log\frac{N\cdot|w_{1,2}\cap c_1|}{|w_{1,2}|\cdot|c_1|} + \sum_{\tau = 2}^K \frac{|w_{1,2}\cap c_\tau|}{N}\log\frac{N\cdot|w_{1,2}\cap c_\tau|}{|w_{1,2}|\cdot|c_\tau|} \notag \\ 
    & = \frac{t_{1}p+t_{2}q}{N}\log\frac{N(t_{1}p+t_{2}q)}{(p+q) s_{1}} + \sum_{j=1}^{(1-t_{1})p+(1-t_{2})q}\frac{1}{N}\log\frac{N \cdot |w_{1,2} \cap c_{i_j}| }{(p+q) s_{i_{j}}} .
    \# 
    Similarly, we have
    \# \label{eq:102}
    \mathrm{(II)} & = \frac{|w_{1}\cap c_1|}{N}\log\frac{N\cdot|w_{1}\cap c_1|}{|w_{1}|\cdot|c_1|} + \sum_{\tau = 2}^K \frac{|w_{1}\cap c_\tau|}{N}\log\frac{N\cdot|w_{1}\cap c_\tau|}{|w_{1}|\cdot|c_\tau|} \notag \\ 
    & = \frac{t_{1}p}{N}\log\frac{N t_{1}}{ s_{1}} + \sum_{j=1}^{(1-t_{1})p}\frac{1}{N}\log\frac{N \cdot|w_{1}\cap c_{i_j}| }{p s_{i_{j}}},
    \# 
    and 
    \# \label{eq:103}
    \mathrm{(III)} & = \frac{|w_{2}\cap c_1|}{N}\log\frac{N\cdot|w_{2}\cap c_1|}{|w_{2}|\cdot|c_1|} + \sum_{\tau = 2}^K \frac{|w_{2}\cap c_\tau|}{N}\log\frac{N\cdot|w_{2}\cap c_\tau|}{|w_{2}|\cdot|c_\tau|} \notag \\ 
    & = \frac{t_{2}q}{N}\log\frac{N t_{2}}{ s_{2}} + \sum_{j=1}^{(1-t_{2})q}\frac{1}{N}\log\frac{N \cdot|w_{2}\cap c_{i_j}| }{q s_{i_{j}}} .
    \# 
    Plugging Eq.~\eqref{eq:101}, Eq.~\eqref{eq:102}, and Eq.~\eqref{eq:103} into Eq.~\eqref{eq:mi}, together with the fact that 
    \$
    |w_{1,2} \cap c_{i_j}| & \ge \max\{|w_{1} \cap c_{i_j}|, |w_{2} \cap c_{i_j}| \}, \quad \forall j \in \{1, 2, \cdots, (1- t_1)p + (1- t_2)q\},
    \$
    we obtain that
    \# 
    \mathbb{I}(\Omega^*; C) &\ge \mathbb{I}(\Omega; C) + \underbrace{\frac{t_{1}p}{N}\log\frac{t_{1}p+t_{2}q}{t_{1}(p+q)} + \frac{t_{2}q}{N}\log\frac{t_{1}p+t_{2}q}{t_{2}(p+q)}}_{\displaystyle\mathrm{(IV)}} \label{eq:104} \\
    & \qquad - \underbrace{\Big(\frac{(1-t_{1})p+(1-t_{2})q}{N}\log (p+q)-\frac{(1-t_{1})p}{N}\log p - \frac{(1-t_{2})q}{N} \log q\Big)}_{\displaystyle\mathrm{(V)}}. \label{eq:105}
    \#

    For Term (IV) in Eq.~\eqref{eq:104}, by the Taylor expansion 
    \$
    \log(1+x) = \sum_{u = 1}^{\infty} \frac{(-1)^{u-1}}{u} \cdot x^u,
    \$ 
    we have
    \# \label{eq:111}
    \mathrm{(IV)} & = \frac{t_1p}{N}\sum_{u = 1}^{\infty}  \frac{(-1)^{u-1}}{u} \cdot \Big( \frac{(t_2 - t_1) q}{ t_1 (p+q)} \Big)^u + \frac{t_2q}{N}\sum_{u = 1}^{\infty}  \frac{(-1)^{u-1}}{u} \cdot \Big[ \frac{(t_1 - t_2) p}{ t_2 (p+q)} \Big]^u \notag \\ 
    & = \sum_{v = 1}^{\infty} \frac{1}{2v - 1} \cdot 
     \Big[ \frac{t_1p}{N} \cdot \Big( \frac{(t_2 - t_1) q}{ t_1 (p+q)} \Big)^{2v-1} + \frac{t_2q}{N}\Big( \frac{(t_1 - t_2) p}{ t_2 (p+q)} \Big)^{2v-1} \Big] \notag \\ 
    & \qquad - \sum_{v=1}^{\infty}  \frac{1}{2v} \cdot 
     \Big[ \frac{t_1p}{N} \cdot \Big( \frac{(t_2 - t_1) q}{ t_1 (p+q)} \Big)^{2v} + \frac{t_2q}{N}\Big( \frac{(t_1 - t_2) p}{ t_2 (p+q)} \Big)^{2v} \Big] .
    \#  
    For ease of presentation, we denote $m = q/p \ge 1$. Then for any $v \ge 1$, we have
    \# \label{eq:112}
    & \sum_{v=1}^{\infty}  \frac{1}{2v} \cdot 
     \Big[ \frac{t_1p}{N} \cdot \Big( \frac{(t_2 - t_1) q}{ t_1 (p+q)} \Big)^{2v} + \frac{t_2q}{N}\Big( \frac{(t_1 - t_2) p}{ t_2 (p+q)} \Big)^{2v} \Big] \notag \\ 
     & \qquad = \sum_{v = 1}^{\infty} \frac{pq(t_1 - t_2)^{2v}}{2vN(p+q)^{2v}} \cdot \Big[ \frac{q^{2v-1}}{t_1^{2v-1}} + \frac{p^{2v-1}}{t_2^{2v-1}}\Big] \notag \\
     & \qquad \le \sum_{v=1}^{\infty} \frac{mp^2(3/10)^{2v}}{2vNp^{2v}(1+m)^{2v}} \cdot (1 + m^{2v-1}) \cdot \Big( \frac{10p}{7} \Big)^{2v-1} \notag \\
     & \qquad \le \frac{3p}{20N} \sum_{v = 1}^{\infty} \frac{1}{v} \cdot \Big( \frac{3}{7}\Big)^{2v-1} \le \frac{0.0716 p}{N},
    \#  
    where the first inequality uses $m = q/p$ and the assumption that $t_1, t_2 \in [0.7, 1]$, the second inequality follows the fact that $m (1 + m^{2v-1}) \le (1+ m)^{2v}$, and the last inequality follows that
    \$
    \sum_{v = 1}^{\infty} \frac{1}{v} \cdot \Big( \frac{3}{7}\Big)^{2v-1} \le \frac{3}{7} + \sum_{v = 2}^{\infty}  \frac{1}{2} \cdot \Big(\frac{3}{7}\Big)^{2v-1} = \frac{267}{560}
    \$ 
    and simple calculations. 

    On the other hand, for any $v \ge 1$, we have 
    \# \label{eq:113}
    & \sum_{v = 1}^{\infty} \frac{1}{2v - 1} \cdot 
     \Big[ \frac{t_1p}{N} \cdot \Big( \frac{(t_2 - t_1) q}{ t_1 (p+q)} \Big)^{2v-1} + \frac{t_2q}{N}\Big( \frac{(t_1 - t_2) p}{ t_2 (p+q)} \Big)^{2v-1} \Big] \notag \\
     & \qquad = \sum_{v = 1}^{\infty} \frac{pq(t_1 - t_2)^{2v-1}}{(2v-1) N (p+q)^{2v-1}} \cdot \Big[ \frac{p^{2v-2}}{t_2^{2v-2}} - \frac{q^{2v-2}}{t_1^{2v-2}}\Big] \notag \\ 
     & \qquad = \sum_{v = 1}^{\infty} \frac{pq(t_1 - t_2)^{2v+1}}{(2v+1) N (p+q)^{2v+1}} \cdot \Big[ \frac{p^{2v}}{t_2^{2v}} - \frac{q^{2v}}{t_1^{2v}}\Big] .
    \#  
    Furthermore, we have
    \# \label{eq:114}
    \sum_{v = 1}^{\infty} \frac{pq(t_1 - t_2)^{2v+1}}{(2v+1) (p+q)^{2v+1}} \cdot \Big[ \frac{p^{2v}}{t_2^{2v}} - \frac{q^{2v}}{t_1^{2v}}\Big] & \ge - \sum_{v = 1}^{\infty} \frac{pq |t_1 - t_2|^{2v+1}}{(2v+1) N (p+q)^{2v+1}} \cdot  \Big[ \frac{p^{2v}}{t_2^{2v}} + \frac{q^{2v}}{t_1^{2v}}\Big]  \notag \\ 
    & \ge - \frac{p}{N} \sum_{v = 1}^{\infty} \frac{1}{2v+1} \cdot \Big(\frac{3}{10}\Big)^{2v+1} \notag \\
    & \ge - \frac{0.0096 p}{N},
    \# 
    where the second inequality uses the facts that $(p+q)^{2v+1} \ge q^{2v+1} + q p^{2v}$ and $t_1, t_2 \in [0.7, 1]$, and the last inequality follows that
    \$
    \sum_{v = 1}^{\infty} \frac{1}{2v+1} \Big(\frac{3}{10}\Big)^{2v+1} \le \frac{9}{1000} + \frac{1}{5} \sum_{v = 2}^{\infty} \Big(\frac{3}{10}\Big)^{2v+1} = \frac{9}{1000} + \frac{243}{455000} < 0.0096.
    \$ 
    Combining Eq.~\eqref{eq:111}, Eq.~\eqref{eq:112}, Eq.~\eqref{eq:113}, and Eq.~\eqref{eq:114}, we obtain that 
    \# \label{eq:115}
    \mathrm{(IV)} \ge - \frac{0.0812 p}{N}.
    \#

    \paragraph{Term (V) in Eq.~\eqref{eq:105}.} For Term (V) in Eq.~\eqref{eq:105}, we have
    \# \label{eq:106}
    \mathrm{(V)} & = \frac{(1-t_{1})p+(1-t_{2})q}{N}\log (p+q)-\frac{(1-t_{1})p}{N}\log p - \frac{(1-t_{2})q}{N} \log q \notag \\ 
    &  = \frac{(1-t_1)p}{N} \log \frac{p+q}{p} + \frac{(1 - t_2)q}{N} \log \frac{p+q}{q}.
    \# 
    Furthermore, we have
    \# \label{eq:108}
    & \frac{(1-t_1)p}{N} \log \frac{p+q}{p} + \frac{(1 - t_2)q}{N} \log \frac{p+q}{q} \\
    & \qquad \le \frac{(1-\min\{t_1, t_2\})p}{N} \log \frac{p+q}{p} + \frac{(1 - \min\{t_1, t_2\})q}{N} \log \frac{p+q}{q} \notag \\ 
    & \qquad = \frac{(1-\min\{t_1, t_2\})}{N} \cdot  \Big[ p \log \frac{p+q}{p} + q \log \frac{p+q}{q}\Big] .
    \#  
    Let 
    \# \label{eq:109}
    \Delta h=\frac{1}{N} \cdot \Big[ p \log \frac{p+q}{p} + q \log \frac{p+q}{q} \Big] .
    \# 
    Combining Eq.~\eqref{eq:106}, Eq.~\eqref{eq:108}, and Eq.~\eqref{eq:109}, we obtain
    \# \label{eq:110}
    \mathrm{(V)} \le (1 - \min\{t_1, t_2\}) \cdot \Delta h.
    \# 

    \paragraph{Putting Together.} When $t_1,t_2\ge 0.7$,
    plugging Eq.~\eqref{eq:115} and Eq.~\eqref{eq:110} into Eq.~\eqref{eq:104} and Eq.~\eqref{eq:105}, we have
    \# \label{eq:116}
    \mathbb{I}(\Omega^*; C) \ge \mathbb{I}(\Omega; C) - \frac{0.0812p}{N} - (1 - \min\{t_1, t_2\}) \cdot \Delta h.
    \#  
    Recall that the $\Delta h$ defined in Eq.~\eqref{eq:109} takes the form
    \# \label{eq:117}
    \Delta h & = \frac{1}{N} \cdot \Big[ p \log \frac{p+q}{p} + q \log \frac{p+q}{q} \Big] \notag \\
    & = \frac{p}{N} \cdot \Big( \log (1+m) + m \log \big(1 + \frac{1}{m}\big) \Big) \notag \\
    & \ge 2\log 2 \cdot \frac{p}{N} , 
    \#
    where the second equality uses $m = q/p$, the last inequality uses $m \ge 1$.  
    Putting Eq.~\eqref{eq:116} and Eq.~\eqref{eq:117} together, we have
    \# \label{eq:118}
    \mathbb{I}(\Omega^*; C) \ge \mathbb{I}(\Omega; C)  - (1.0586 - \min\{t_1, t_2\}) \cdot \Delta h .
    \# 
    

    Then, we calculate the entropy after fusion.
    \begin{align}\label{eq:ent}
        \HH(\Omega^{*})&=\HH(\Omega)-\frac{p+q}{N}\log\frac{p+q}{N}+\frac{p}{N}\log\frac{p}{N}+\frac{q}{N}\log\frac{q}{N} \notag \\
        &=\HH(\Omega)-\frac{p+q}{N}\log (p+q)+\frac{p}{N}\log p+\frac{q}{N}\log q \notag \\
        &=\HH(\Omega) -\Delta h
    \end{align}
  
    Recall that 
    \$
    n_1 = \frac{2\mathbb{I}(\Omega; C)}{\HH(\Omega)+\HH(C)}, \quad n_2 = \frac{2\mathbb{I}(\Omega^{*}; C)}{\HH(\Omega^{*})+\HH(C)} .
    \$ 
    By Eq.~\eqref{eq:118} and Eq.~\eqref{eq:ent}, we know that the sufficient condition of $n_2 \ge n_1$ is
    \$
     \frac{2[\mathbb{I}(\Omega; C)-(1.0586-\min\{t_1, t_2\}) \cdot \Delta h]}{\HH(\Omega)+\HH(C)-\Delta h}\geq \frac{2\mathbb{I}(\Omega; C)}{\HH(\Omega)+\HH(C)},
    \$ 
    which is equivalent to 
    \$
    n_1 = \frac{2\mathbb{I}(\Omega; C)}{\HH(\Omega)+\HH(C)} \ge 2 \cdot (1.0586-\min\{t_1, t_2\}),
    \$
    which concludes the proof of Theorem~\ref{lemma:purity}.

    
\end{proof}

\subsection{Derivation of Eq.~\eqref{prob:ori}}
\label{appendix:pf:merge_prob}
\begin{proof}[Derivation of Eq.~\eqref{prob:ori}]
     We consider the queried result of ``must-link'' ($c_m=c_n$ ) between $w_i$ and $w_j$ as a conditional event, which is expressed by $\{\forall s\in w_i, \forall t\in w_j, e_{st}=1\mid\forall (s,t)\in w_i \text{ or } w_j, e_{st}=1\}$. Here the condition implies that samples within a cluster are assigned to the same class. Further, we follow the setup in probabilistic clustering~\cite{lu2004semi,liu2022mpc} that models the distribution of clustering based on the pairwise probability. Then, the joint probability density of a clustering $\pi= [z_1,z_2,\cdots,z_m$] for $m$ samples is expressed as $\PP(\pi)=\frac{1}{\alpha}\prod_{s,t\in[1,2,\cdots,m]}\PP(e_{st}=1)^{I(z_s=z_t)}\times \PP(e_{st}=0)^{I(z_s\neq z_t)}$, where $I(\cdot)$ is the indicator function, and $\alpha$ is the normalization factor.
    
     Under the condition of $\forall (s,t)\in w_i \text{ or } w_j, e_{st}=1$, the two events $\{\forall s\in w_i, \forall t\in w_j, e_{st}=1\}$ and $\{\forall s\in w_i, \forall t\in w_j, e_{st}=0\}$ are mutually exclusive. And we denote them as $y(w_i)=y(w_j)$ and $y(w_i)\neq y(w_j)$ for simplicity. Therefore, by the formula of conditional probability, we can obtain:
    \begin{align*} 
    \PP(c_m=c_n)
        &=\PP(\forall s\in w_i, \forall t\in w_j, e_{st}=1 \mid \forall (s,t)\in w_i \text{ or } w_j, e_{st}=1)\\
        &=\frac{\PP(y(w_i)=y(w_j), \forall (s,t)\in w_i \text{ or } w_j, e_{st}=1)}{\PP(\forall (s,t)\in w_i \text{ or } w_j, e_{st}=1))} \\ 
        &=\frac{\PP(y(w_i)=y(w_j), \forall (s,t)\in w_i \text{ or } w_j, e_{st}=1))}{\PP(y(w_i)=y(w_j), \forall (s,t)\in w_i \text{ or } w_j, e_{st}=1))+\PP(y(w_i)\neq y(w_j), \forall (s,t)\in w_i \text{ or } w_j, e_{st}=1))} \notag \\
        &=\frac{\frac{1}{A}\prod_{s\in w_i,t\in w_j} \PP(e_{st}=1)\prod_{s,t\in w_i, s,t\in w_j}\PP(e_{st}=1)}{\frac{1}{A}[\prod_{s\in w_i,t\in w_j} \PP(e_{st}=1) + \prod_{s\in w_i,t\in w_j} \PP(e_{st}=0)] \prod_{s,t\in w_i, s,t\in w_j}\PP(e_{st}=1)} \notag \\
        &=\frac{\prod_{s\in w_i,t\in w_j} \PP(e_{st}=1)}{\prod_{s\in w_i,t\in w_j} \PP(e_{st}=1)+\prod_{s\in w_i,t\in w_j} \PP(e_{st}=0)},
        \end{align*}
    where $A$ is the normalization factor of the joint probability density for samples in $\{w_i,w_j\}$.
\end{proof}

\textbf{Remark on the condition.} The assignment of classes for samples is subject to clustering. For instance, to categorize a group of sheep, various attributes like gender, age, health, and weight can be the basis for clustering. In our scenario to measure the likelihood of ``must-link'', what we indeed care about is the probability that the two clusters should be merged during the clustering, rather than the actual classes of these samples (this is the mission of oracles). After all, each class assignment corresponds to a specific meaning in real applications.

\textbf{Remark on the aggregation probability.} Previous studies often evaluate the likelihood of $c_m=c_n$ with heuristic strategies. Common strategies include using the distance between the central samples of two clusters or the closest distance between the clusters themselves to determine which cluster pairs to query. However, these methods that focus on a single sample from a cluster, may become less effective for clusters with irregular shapes, and Eq. ~\eqref{prob:ori} provides a better estimation that takes into account the influence of all samples.

\textbf{Remark on the query style.} A3S requires oracles to provide pairwise comparison results for each selected sample pair, which demands minimal domain-specific knowledge and is easier to implement \citep{xiong2016active}, especially when the number of classes is large. In contrast, traditional active learning necessitates that the oracle assigns specific labels to the selected samples or annotates them according to predefined rules \citep{deng2023counterfactual}, causing much heavier costs during the annotation process.

\subsection{Justification of Approximation}
\label{appendix:coro}

\textbf{Regarding the Approximation of Entropy. }   
If $(p+q)\ll N$, then 
\#
\Delta h &= \frac{p+q}{N}\log(p+q)-\frac{p}{N}\log p-\frac{q}{N}\log q \notag \\
&< \frac{p+q}{N}\log(p+q) \notag \\
&< \frac{p+q}{N}\log\frac{N}{p+q}. \notag
\#
Note that $\HH(\Omega)=\sum_{s} \frac{s}{N}\log\frac{N}{s}$, where $s$ is the cluster size like $p$ and $q$. Together with $p+q\ll N$, we have 
\$
\Delta h \ll \HH(\Omega) < \HH(\Omega) + \HH(C).
\$

\textbf{Regarding the Mutual Information. }
If the purity of $w_1$ and $w_2$ is 1, and they belong to the same class $c_\tau$, then we have
\$
\PP(w_1\cap c_\tau)=\PP(w_1),\quad  \PP(w_2\cap c_\tau)=\PP(w_2), \quad \PP(w_{1,2}\cap c_\tau)=\PP(w_{1,2})=\PP(w_1)+\PP(w_2).
\$

Hence, and we have $\mathbb{I}(\Omega^{*}; C) =\mathbb{I}(\Omega; C)$. This is because    
\#
\mathbb{I}(\Omega^{*}; C) & =\mathbb{I}(\Omega; C) + \PP(w_{1,2}\cap c_\tau)\log\frac{\PP(w_{1,2}\cap c_\tau)}{\PP(w_{1,2})\PP(c_\tau)} \notag\\
& \qquad - \PP(w_{1}\cap c_\tau)\log\frac{\PP(w_{1}\cap c_\tau)}{\PP(w_{1})\PP(c_\tau)} - \PP(w_{2}\cap c_\tau)\log\frac{\PP(w_{2}\cap c_\tau)}{\PP(w_{2})\PP(c_\tau)} \notag \\
 & =\mathbb{I}(\Omega; C) + \PP(w_{1,2})\log\frac{1}{\PP(c_\tau)}-\PP(w_1)\log\frac{1}{\PP(c_\tau)}-\PP(w_2)\log\frac{1}{\PP(c_\tau)}\notag\\
 &=\mathbb{I}(\Omega; C). \notag
\#

When the purity is less than 1.0, considering Eq.~\eqref{eq:118}, and fact that $\Delta h \ll \min\{\HH(\Omega),\HH(C)\}$, we have
\#
|\mathbb{I}(\Omega^*; C) - \mathbb{I}(\Omega; C)| < (1.0586 - \min\{t_1, t_2\}) \cdot \Delta h \ll \min\{\HH(\Omega),\HH(C)\}.
\#

\subsection{Fast Transitive Inference}
\label{appendix:exp:fti}
FTI involves expanding the set of constraints based on the information within the original set. For example, if $(x_1,x_2)$ and $(x_2,x_3)$ are must-link constraints, and $(x_1,x_4)$ is a cannot-link constraint, it implies that $(x_1,x_3)$ must be a must-link constraint, while $(x_2,x_4)$ and $(x_3,x_4)$ must be cannot-link constraints.

To facilitate the process, we present an efficient method, Fast Transitive Inference (FTI), which is designed to discover the transitive closure for the constraint set in \Name{}. The implementation is shown in Algorithm~\ref{algo:fti}. 

The performance of FTI is guaranteed by Theorem~\ref{lemma:tic}.

\begin{theorem}[Completeness of FTI] 
\label{lemma:tic}
By executing the FTI algorithm every time a new human query is made, we can always get the latest transitive closure.
\end{theorem}

\begin{proof}[Proof of Theorem \ref{lemma:tic}]
    Suppose the must-link sample sets with samples $i$ and $j$ are denoted as $G_i$ and $G_j$, respectively. And the sample sets that are cannot-link with $i$ and $j$ are denoted as $g_i$ and $g_j$. When the constraints between $i$ and $j$ are queried, only the constraints of sample pairs within $\{G_{i}, G_{j}, g_{i}, g_{j}\}$ may change, as no must-link constraints are built between them and the rest samples. We discuss the two cases where $i$ and $j$ are must-linked or cannot-linked:
    \begin{itemize}[leftmargin=*]
        \item[1] $(i,j)$ is must-linked. 
        First, FTI updates the constraints for sample pairs related to $i$, then $G_{i}^{\prime}=G_{i}\cup \{j\}$, $G_{j}^{\prime}=G_{j}\cup G_{i}$, and $g_{j}^{\prime}=g_{j}\cup g_{i}$. The constraints between $i$ and $G_j,g_j$ have not yet been updated. Then, FTI updates the constraints for sample pairs related to $j$, then we have $G_{i}^{\prime}=G_{i}\cup G_{j}$ and $g_{i}^{\prime}=g_{i}\cup g_{j}$. And this means all ml constraints between $G_{i}$ and $G_{j}$, and cl constraints between $\{G_{i}, G_{j}\}$ and $\{g_{i}, g_{j}\}$ are updated and stored in the state matrix $S$.

        \item[2] $(i,j)$ is cannot-linked.
        First, FTI updates the constraints for sample pairs related to $i$, then $g_{i}^{\prime}=g_{i}\cup \{j\}$ and $g_{j}^{\prime}=g_{j}\cup G_{i}$. Then FTI updates the constraints for sample pairs related to $j$, then we have $g_{i}^{\prime}=g_{i}\cup G_{j}$. And this means all cl constraints between $G_{i}$ and $G_{j}$ are updated and stored in $S$.
    \end{itemize}
    Combining these two scenarios, we finish the proof of Theorem \ref{lemma:tic}
\end{proof}

\section{Implementation Details of \Name{}}
\label{appendix:imple}
\subsection{Estimating Pairwise Probability} \label{appendix:prob-esti}
Following the setup in \cite{liu2022mpc}, we use isotonic regression to learn a regressor that estimates the pairwise posterior probability $\PP(e_{ij}=1|d_{ij})$, where $d_{ij}$ is the Euclidean distance between samples $i$ and $j$. The estimation encompasses three steps: 
\begin{itemize}[leftmargin=*]
\item[(1)] Since ground truth labels are unavailable, we employ K-means clustering to generate pseudo labels for the samples. Alternatively, Fast Probabilistic Clustering (FPC) can also be used for this purpose, where similarity values between samples serve as a rudimentary approximation of pairwise probabilities. In practical applications, cosine similarity is particularly well-suited for FPC.
\item[(2)] Next, we generate the training data for isotonic regression. We utilize the k-nearest neighbors of each sample to form sample pairs. The independent variable of isotonic regression is the Euclidean distance between two samples in the sample pair. Each pair is labeled as 0 or 1, indicating whether the two samples in each pair share the same pseudo label. The label serves as the dependent variable of isotonic regression.
\item[(3)] Finally, we conduct isotonic regression on the gathered data, learning a function that maps the Euclidean distance between two samples to the pairwise probability, i.e., $\PP(e_{st}=1|d_{st})$.
\end{itemize}

\citet{liu2022mpc} also proposed Graph-context-aware refinement to enhance the quality of the posterior probability, but it is not an essential component of \Name{}. Therefore, we did not include it in our experiments. However, incorporating them would further enhance the performance of \Name{}, as they can improve the quality of the estimated merging probability between cluster pairs.

\subsection{Calculating Aggregation Probability}
We derive the aggregation probability based on the condition that the purity of two clusters are 1.0. 
In practical applications, the assumption that both clusters have a purity of 1.0 might not strictly hold. Specifically, the pairwise probability between major samples and outlier samples from $w_i$ and $w_j$ will degrade the aggregation probability (i.e., making the result biased towards 0). To deal with this issue, 
we propose a variant of $\PP(c_m=c_n)$, denoted as $\PP_{\knn}(c_m=c_n)$. This variant considers only the sample pairs within the k-nearest neighbors. Assume that $|w_i|\le|w_j|$, the formulation is as follows:
\small
\begin{equation}\label{prob:mul}
    \PP_{\knn}(c_m=c_n)=\frac{\prod\limits_{s\in w_i  t\in \knn_{w_j}(s)} \PP(e_{st}=1)}{\prod\limits_{s\in w_i  t\in \knn_{w_j}(s)} \PP(e_{st}=1) + \prod\limits_{s\in w_i  t\in \knn_{w_j}(s)} \PP(e_{st}=0)}, 
\end{equation}
\normalsize
where $\knn_{w_j}(s)$ denotes the nearest neighbors of $s$ in $w_j$. Empirical results show that \Name{} is not sensitive to the number of neighbors, and we consider 4 neighbors for each sample in our experiments.

\subsection{Feature Extraction}
The facial characteristics in the \texttt{Humbi-Face} dataset are extracted using a face recognition model. Similarly, the \texttt{MK20} and \texttt{MK100} datasets utilize a person re-identification model for body feature extraction \citep{liu2022mpc}. In the case of \texttt{MS1M-10k} and \texttt{MS1M-100k}, facial features are extracted using the Arcface model~\citep{deng2019arcface}. The \texttt{Handwritten} dataset encompasses four types of features: average pixel features, Fourier coefficient features, Zernike moments features, and Karhunen-Loève coefficient features.

\subsection{Hyperparameter Setting}
The implementation of \Name{} involves two hyperparameters: the threshold for density test and the number of neighbors considered in Fast Probabilistic Clustering. We report our choice of these two parameters in Table~\ref{tab:threshold}. In particular, the hyper-parameter $\tau$ is used to filter out the clusters with low density, and it is set to be slightly lower than the average density among all clusters. For the selection of $\tau$, we first compute density using the formula in Eq. \eqref{eq:dt} for all clusters, then calculate the mean value as $d$, and set $\tau$ as $d-0.1$. The final results of A3S are not sensitive to this value, and perturbing it to $d$ or $d-0.05$ has a negligible impact on the final clustering result.
\begin{table}[ht]
	\centering
	\caption{Hyperparameter setting of \Name{}. }
	\resizebox{0.8\textwidth}{!}{
		\begin{tabular}{c|cccccc}
			\hline 
                dataset & \texttt{MK20} & \texttt{MK100} & \texttt{Handwritten} & \texttt{Humbi-Face} & \texttt{MS1M-10k} & \texttt{MS1M-100k} \\
			$\tau$ & 0.5 & 0.8 & 0.8 & 0.5 & 0.5 & 0.5 \\
                \hline 
                neighbors & 50 & 50 & 50 & 50 & 50 & 50 \\
			\hline 
	\end{tabular}}
	\label{tab:threshold}
	\vspace{-20pt}
\end{table}

\subsection{Computing Resources}
We utilize a [GeForce RTX 3090 Ti] for feature extraction using DNN models. For the implementation of baseline methods and \Name{}, we perform the experiments on a machine equipped with an Intel(R) Xeon(R) Platinum 8163 CPU @ 2.50GHz.

\section{More Experiment Results}
\label{appendix:exp_result}
In our ablation study, we examine the impact of varying the adaptive cluster number on \Name{}. We define this number as $r \cdot k$, where $k$ is the adaptive number determined by FPC and $r$ is a ratio factor set to values in the set {1, 1.5, 2}. Utilizing the K-means algorithm, we generate initial clustering with the cluster number $r\cdot k$, and present the corresponding results of \Name{} in Figure~\ref{fig:adapt_abl}. It's observed that the performance of the initial clustering is relatively sensitive to the adaptive cluster number. However, the selection of this adaptive number has a minimal effect on the number of queries required to achieve the desired clustering outcome. This demonstrates \Name{}'s robustness to variations in the adaptive cluster number.

\begin{figure*}[h]
	\centering
        \subfigure{
        \begin{minipage}[t]{0.24\linewidth}
            \centering
            \includegraphics[width=1\textwidth]{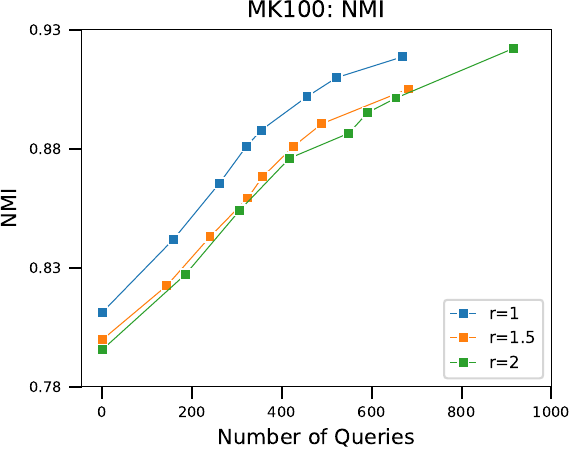}\\
        \end{minipage}%
        }
        \subfigure{
        \begin{minipage}[t]{0.24\linewidth}
            \centering
            \includegraphics[width=1\textwidth]{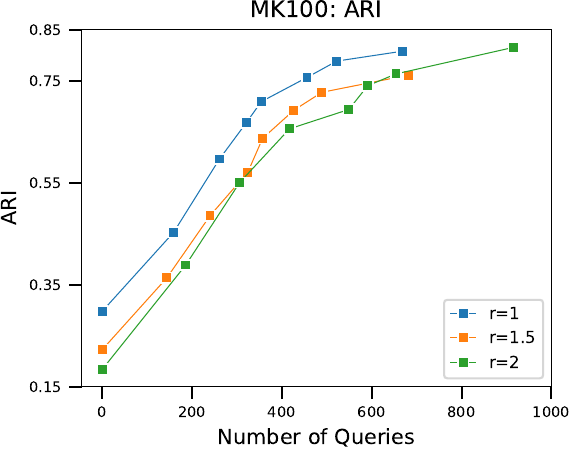}\\
        \end{minipage}%
        }
        \subfigure{
        \begin{minipage}[t]{0.24\linewidth}
            \centering
            \includegraphics[width=1\textwidth]{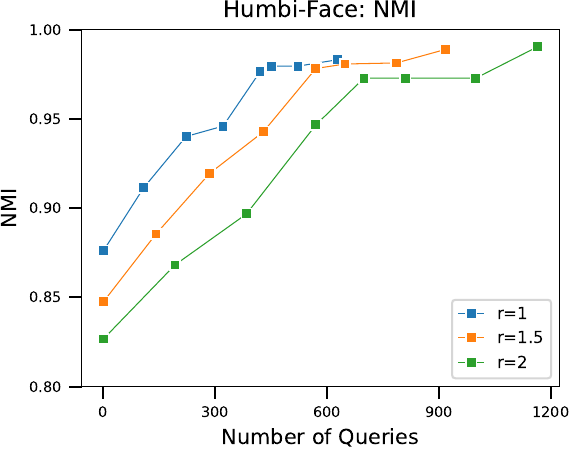}\\
        \end{minipage}%
        }
        \subfigure{
        \begin{minipage}[t]{0.24\linewidth}
            \centering
            \includegraphics[width=1\textwidth]{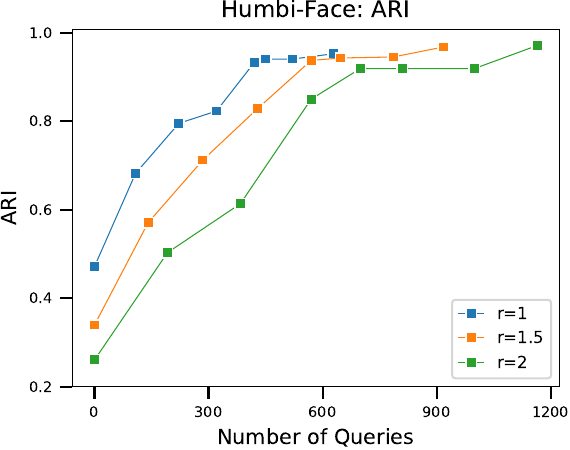}\\
        \end{minipage}%
        }
        \vspace{-10pt}
	\caption{The NMI and ARI performance of \Name{} when the adaptive cluster number is set as $r\cdot k$, where $k$ is the adaptive number generated by FPC and $r$ is the ratio factor.
	}
	\label{fig:adapt_abl}
	\vspace{-12pt}
\end{figure*}